\documentclass{article}
\usepackage{ijcai21}

\usepackage{times}
\usepackage{soul}
\usepackage{url}
\usepackage[hidelinks]{hyperref}
\usepackage[utf8]{inputenc}
\usepackage[small]{caption}
\usepackage{graphicx}
\usepackage{amsmath}
\usepackage{amsthm}
\usepackage{booktabs}
\usepackage{algorithm}
\usepackage{algorithmic}
\usepackage[utf8]{inputenc}
\usepackage[T1]{fontenc}
\usepackage{hyphenat}
\usepackage{xspace}
\usepackage{amsfonts}
\usepackage{hyperref}
\usepackage{url}
\usepackage{booktabs}
\usepackage{multirow}
\usepackage{makecell}
\usepackage{caption}
\usepackage{minibox}
\usepackage{bbm}
\usepackage{graphicx}
\usepackage{balance}
\usepackage{mathtools}
\usepackage{color}
\usepackage{marvosym}
\usepackage{ifthen}
\usepackage{textcomp}
\usepackage{enumitem}
\usepackage{verbatim}
\usepackage{algorithm}
\usepackage{algorithmic}
\usepackage{numprint}
\usepackage{balance}

\usepackage{amsthm}
\theoremstyle{plain}
\newcommand{\chatoDisplayMode}[1]{#1}

\definecolor{MyRed}{rgb}{0.6,0.0,0.0} 
\definecolor{MyBlack}{rgb}{0.1,0.1,0.1} 
\newcommand{\inred}[1]{{\color{MyRed}\sf\textbf{\textsc{#1}}}}
\newcommand{\frameit}[2]{
  \begin{center}
  {\color{MyRed}
  \framebox[.9\columnwidth][l]{
    \begin{minipage}{.85\columnwidth}
    \inred{#1}: {\sf\color{MyBlack}#2}
    \end{minipage}
  }\\
  }
  \end{center}
}

\newcommand{\note}[2][]{\chatoDisplayMode{\def\@tmpsig{#1}\frameit{{\Pointinghand} Note}{#2\ifx \@tmpsig \@empty \else \mbox{ --\em #1}\fi}}}
\newcommand{\todo}[2][]{\chatoDisplayMode{\def\@tmpsig{#1}\frameit{{\Writinghand} To-do}{#2\ifx \@tmpsig \@empty \else \mbox{ --\em #1}\fi}}}

\newcommand{\abbrevStyle}[1]{#1}
\newcommand{\ie}{\abbrevStyle{i.e.}\xspace}
\newcommand{\eg}{\abbrevStyle{e.g.}\xspace}
\newcommand{\cf}{\abbrevStyle{cf.}\xspace}

\newcommand{\Secref}[1]{Sec.~\ref{#1}}
\newcommand{\Eqnref}[1]{Eq.~\ref{#1}}

\newcommand{\Tabref}[1]{Table~\ref{#1}}
\newcommand{\Figref}[1]{Fig.~\ref{#1}}

\newcommand{\Appref}[1]{Appendix~\ref{#1}}
\newcommand{\Thmref}[1]{Thm.~\ref{#1}}

\newcommand{\xhdr}[1]{\vspace{1.7mm}\noindent{{\bf #1.}}}

\newcommand{\textcite}[1]{\citeauthor{#1} \shortcite{#1}}

\newcommand{\hide}[1]{}

\makeatletter
\newcommand{\iffont}[2]{\ifthenelse{\equal{\f@family}{#1}}{#2}{}}
\makeatother

  \usepackage{mathptmx}

  \DeclareSymbolFont{greek}{OML}{cmm}{m}{n}
  \DeclareMathSymbol{\alpha}{\mathalpha}{greek}{"0B}
  \DeclareMathSymbol{\beta}{\mathalpha}{greek}{"0C}
  \DeclareMathSymbol{\gamma}{\mathalpha}{greek}{"0D}
  \DeclareMathSymbol{\delta}{\mathalpha}{greek}{"0E}
  \DeclareMathSymbol{\epsilon}{\mathalpha}{greek}{"0F}
  \DeclareMathSymbol{\zeta}{\mathalpha}{greek}{"10}
  \DeclareMathSymbol{\eta}{\mathalpha}{greek}{"11}
  \DeclareMathSymbol{\theta}{\mathalpha}{greek}{"12}
  \DeclareMathSymbol{\iota}{\mathalpha}{greek}{"13}
  \DeclareMathSymbol{\kappa}{\mathalpha}{greek}{"14}
  \DeclareMathSymbol{\lambda}{\mathalpha}{greek}{"15}
  \DeclareMathSymbol{\mu}{\mathalpha}{greek}{"16}
  \DeclareMathSymbol{\nu}{\mathalpha}{greek}{"17}
  \DeclareMathSymbol{\xi}{\mathalpha}{greek}{"18}
  \DeclareMathSymbol{\pi}{\mathalpha}{greek}{"19}
  \DeclareMathSymbol{\rho}{\mathalpha}{greek}{"1A}
  \DeclareMathSymbol{\sigma}{\mathalpha}{greek}{"1B}
  \DeclareMathSymbol{\tau}{\mathalpha}{greek}{"1C}
  \DeclareMathSymbol{\upsilon}{\mathalpha}{greek}{"1D}
  \DeclareMathSymbol{\phi}{\mathalpha}{greek}{"1E}
  \DeclareMathSymbol{\chi}{\mathalpha}{greek}{"1F}
  \DeclareMathSymbol{\psi}{\mathalpha}{greek}{"20}
  \DeclareMathSymbol{\omega}{\mathalpha}{greek}{"21}
  \DeclareMathSymbol{\varepsilon}{\mathalpha}{greek}{"22}
  \DeclareMathSymbol{\vartheta}{\mathalpha}{greek}{"23}
  \DeclareMathSymbol{\varpi}{\mathalpha}{greek}{"24}
  \DeclareMathSymbol{\varrho}{\mathalpha}{greek}{"25}
  \DeclareMathSymbol{\varsigma}{\mathalpha}{greek}{"26}
  \DeclareMathSymbol{\varphi}{\mathalpha}{greek}{"27}
  \DeclareSymbolFont{otone}{OT1}{cmr}{m}{n}
  \DeclareMathSymbol{\Gamma}{\mathalpha}{otone}{0}
  \DeclareMathSymbol{\Delta}{\mathalpha}{otone}{1}
  \DeclareMathSymbol{\Theta}{\mathalpha}{otone}{2}
  \DeclareMathSymbol{\Lambda}{\mathalpha}{otone}{3}
  \DeclareMathSymbol{\Xi}{\mathalpha}{otone}{4}
  \DeclareMathSymbol{\Pi}{\mathalpha}{otone}{5}
  \DeclareMathSymbol{\Sigma}{\mathalpha}{otone}{6}
  \DeclareMathSymbol{\Upsilon}{\mathalpha}{otone}{7}
  \DeclareMathSymbol{\Phi}{\mathalpha}{otone}{8}
  \DeclareMathSymbol{\Psi}{\mathalpha}{otone}{9}
  \DeclareMathSymbol{\Omega}{\mathalpha}{otone}{10}
  \DeclareSymbolFont{syms}{OML}{cmm}{m}{it}
  \DeclareMathSymbol{\partial}{\mathord}{syms}{"40}
  \DeclareMathAlphabet{\mathbold}{OML}{cmm}{b}{it}
  \DeclareSymbolFont{largesymbols}{OMX}{cmex}{m}{n}

\usepackage{subcaption}
\usepackage{amssymb}
\usepackage{bm}
\usepackage{bbm}
\usepackage{amsfonts}       % blackboard math symbols
\usepackage{nicefrac}       % compact symbols 
\usepackage{microtype}      % microtypography
\usepackage{hyphenat}
\usepackage{latexsym}
\usepackage{mathptmx}
\usepackage{mathtools}
\usepackage{thmtools} 
\usepackage{thm-restate}
\usepackage{tikz}

\usetikzlibrary{shapes,decorations,arrows,calc,arrows.meta,fit,positioning}
\tikzset{
    -Latex,auto,node distance =1 cm and 1 cm,semithick,
    state/.style ={ellipse, draw, minimum width = 0.7 cm},
    point/.style = {circle, draw, inner sep=0.04cm,fill,node contents={}},
    bidirected/.style={Latex-Latex,dashed},
    el/.style = {inner sep=2pt, align=left, sloped}
}

\newcommand{\bA}{\mathbf{A}}

\newcommand{\bE}{\mathbf{E}}
\newcommand{\bN}{\mathbf{N}}
\newcommand{\be}{\mathbf{e}}
\newcommand{\bn}{\mathbf{n}}
\newcommand{\bW}{\mathbf{W}}
\newcommand{\bX}{\mathbf{X}}
\newcommand{\bY}{\mathbf{Y}}
\newcommand{\bZ}{\mathbf{Z}}

\newcommand{\bF}{\mathbf{F}}
\newcommand{\bI}{\mathbf{I}}

\newcommand{\bi}{\mathbf{i}}
\newcommand{\bPA}{\mathbf{PA}}

\DeclareMathOperator{\DO}{do}
\DeclareMathOperator{\OD}{OD}
\DeclareMathOperator{\ID}{ID}
\DeclareMathOperator{\CD}{CD}
\DeclareMathOperator{\SID}{SID}
\DeclareMathOperator{\SHD}{SHD}
\newcommand\independent{\protect\mathpalette{\protect\independenT}{\perp}}
\def\independenT#1#2{\mathrel{\rlap{$#1#2$}\mkern2mu{#1#2}}}

\newcommand{\newcite}[1]{\citeauthor{#1}\ \shortcite{#1}}

\newtheorem*{theorem*}{Theorem}

\title{A Ladder of Causal Distances}

\author{
Maxime Peyrard
\And
Robert West
\affiliations
EPFL
\emails
\{maxime.peyrard, robert.west\}@epfl.ch
}
\begin{document}

\maketitle

\begin{abstract}
Causal discovery, the task of automatically constructing a causal model from data, is of major significance across the sciences.
Evaluating the performance of causal discovery algorithms should ideally involve comparing the inferred models to ground-truth models available for benchmark datasets, which in turn requires a notion of distance between causal models.
While such distances have been proposed previously, they are limited by focusing on graphical properties of the causal models being compared.
Here, we overcome this limitation by defining distances derived from the causal distributions induced by the models, rather than exclusively from their graphical structure.
Pearl and Mackenzie [2018] have arranged the properties of causal models in a hierarchy called the ``ladder of causation'' spanning three rungs: observational, interventional, and counterfactual.
Following this organization, we introduce a hierarchy of three distances, one for each rung of the ladder.
Our definitions are intuitively appealing as well as efficient to compute approximately.
We put our causal distances to use by benchmarking standard causal discovery systems on both synthetic and real-world datasets for which ground-truth causal models are available.
% The proposed causal distances can have further applications in the field of causality.
% Finally, we highlight the usefulness of our causal distances by briefly discussing further applications beyond the evaluation of causal discovery techniques.
\end{abstract}
\section{Introduction}
\label{intro}
Reasoning about the causes and effects driving physical and societal phenomena is an important goal of science.
Causal reasoning facilitates the prediction of intervention outcomes and can ultimately lead to more principled policymaking~\cite{spirtes2000causation,PearlMackenzie18}.

%% Ladder of causation
Given a causal model, reasoning about cause and effect corresponds to formulating causal queries, which have been organized by \newcite{PearlMackenzie18} in a three-level hierarchy termed the ``ladder of causation'':
\emph{observational} queries correspond to seeing and observing; \emph{interventional} queries correspond to acting and intervening; and \emph{counterfactual} queries correspond to imagining, reasoning, and understanding.

% \begin{enumerate}
%     \item \emph{Observational} queries: seeing and observing. What can we tell about $Y$ if we observe $X=x$?
%     \item \emph{Interventional} queries: acting and intervening. What can we tell about $Y$ if we do $X=x$?
%     \item \emph{Counterfactual} queries: imagining, reasoning, and understanding. Given that $E=e$ actually happened, what would have happened to $Y$ had we done $X=x$?
% \end{enumerate}

Asking such questions requires a causal model to begin with.
Inferring a causal model from observational, interventional, or mixed data is the problem called \emph{causal discovery}.
Much of science is concerned with causal discovery, and automating the task has been receiving increased attention in the machine learning community~\cite{Peters2017}, where causal models have become the tool of choice for tackling important problems such as transfer learning, generalization beyond spurious correlations~\cite{RojSchTurPet18}, and algorithmic fairness~\cite{NIPS2017_6995}.
% , and interpretability~\cite{Lipton:2018}.

Evaluating causal discovery algorithms requires comparing the inferred causal models to ground-truth models on benchmark datasets, which in turn requires a notion of distance between causal models.
When defining such a distance, it is not sufficient to rely on tools developed for comparing standard generative models, such as goodness of fit, as these tools only operate on the first, observational level of the ladder of causation. Remarkably little research has been done on the topic of evaluating and comparing arbitrary causal models higher up the ladder. Existing works focus on specific aspects, such as the outcome of a limited number of interventions manually selected in advance~\cite{singh2017comparative} or the graph structure of the models being compared~\cite{SID}.
The latter work, which proposed the \emph{structural intervention distance} ($\SID$), one of the most prominent causal distance measures, further assumes that the two models have an identical observational joint distribution.
Unfortunately this assumption rarely holds in practice, and we show that even if the observational joint distributions are just slightly different, $\SID$ cannot be trusted. Furthermore, previous causal distances do not cover the counterfactual level.

% In this work, we address the problem of evaluating the causal properties of the learned causal models in comparison to reference causal models.

%%%%%%%%%%%%%%%%%%%%%%%%%

%% Open problem:
% An open problem hindering progress concerns the ability to compare two causal models with respect to their answers to causal queries, e.g., comparing the learned causal model and the reference causal model.
% To evaluate the causal properties of the learned models, it is not sufficient to rely on tools developed to compare standard generative models like goodness of fit because they only refer to the first level of the ladder. Surprisingly, little research has been done on the topic of evaluating and comparing arbitrary causal models higher up the ladder. Existing works focus on specific aspects like comparing only the causal graphs~\cite{SID} or the outcome of few interventions manually selected in advance~\cite{singh2017comparative}. For instance, the most relevant metric to evaluate causal discovery techniques is the \emph{structural intervention distance} ($\SID$), which counts the number of interventional distributions that differ between two causal models under the assumption that they agree on the underlying observational joint distribution. Unfortunately, this assumption rarely holds in practice, and we show that even if the joint distributions are just slightly different, $\SID$ cannot be trusted. Furthermore, $\SID$ does not cover the counterfactual level.

%% Contribution
To close this gap, we introduce three distances (\Secref{distances}), one for each rung of the ladder of causation. Our distances measure the difference between causal models for each type of causal query (observational, interventional, counterfactual).
% on a continuous scale.
Each distance builds upon the distance one level below, thus mirroring the hierarchy of the ladder of causation. We highlight theoretical properties of the distances in relation to previously proposed distances (\Secref{properties}).
% and discuss how to efficiently approximate them in practice (\Secref{approximation}). 
Then, we study their behavior in a series of experiments and put them to use in evaluating existing causal discovery systems (\Secref{experiments}). 
% We conclude with a discussion of implications and further applications (\Secref{future_work}).
Code for reproducing our experiments\footnote{ \url{https://github.com/epfl-dlab/causal-distances}}
and an extended version of the paper (with added appendices)%
\footnote{
\label{fn:app}
\url{https://arxiv.org/abs/2005.02480}}
are available online.

\section{Preliminaries}
\label{preliminaries}
\subsection{Causal Graphs}
We consider a finite ordered set of random variables $\bX = \{X_1, \dots, X_d\}$.
A directed graph $\mathcal{G} = (\bX, \bA)$ consists of the set of indexed nodes $\bX$ together with a set of directed edges $\bA \subseteq \bX \times \bX$. If $(X_i,X_j) \in \bA$, we say that $X_i$ is a \emph{parent} of $X_j$ and denote the set of all parents of $X_j$ with $\bPA_j$.
If $\mathcal{G}$ contains no directed cycle it is called a \emph{directed acyclic graph} (DAG).

DAGs are often used to encode causal assumptions by viewing an edge $(X_i,X_j)$ as the statement ``$X_i$ is a direct cause of $X_j$''~\cite{Pearl:2009}. A graph associated with such causal interpretation is called a \emph{causal graph}.

\begin{figure*}[t!]
     \centering
     \begin{subfigure}[t]{0.32\textwidth}
         \centering
            \begin{tikzpicture}[node distance =0.95 cm and 0.95 cm]
                \node (A) [label = above right:A, point];
                \node (B) [label = left:B, below left = of A, point];
                \node (C) [label = right:C, point, below right = of A];
                %\node (D) [label = below:D, point, below right = of B];
                
                \node (NA) [label = above:$N_A$, above = of A, yshift=-6mm, xshift=0mm, point];
                \node (NB) [label = above:$N_B$, above = of B, yshift=-6mm, xshift=0mm, point];
                \node (NC) [label = above:$N_C$, above = of C, yshift=-6mm, xshift=0mm, point];
                %\node (ND) [label = left:$N_D$, above left = of D, yshift=-6mm, xshift=6mm, point];
                %\node (NB) [label = left:B, below left = of A, point];
                %\node (NC) [label = right:C, point, below right = of A];
                %\node (ND) [label = below:D, point, below right = of B];
        
                \path (A) edge (B);
                \path (A) edge (C);
                \path (B) edge (C);
                %\path (C) edge (D);
                
                \path[dashed] (NA) edge (A);
                \path[dashed] (NB) edge (B);
                \path[dashed] (NC) edge (C);
               % \path[dashed] (ND) edge (D);
            \end{tikzpicture}
         \caption{$\mathfrak{C}$}
         \label{fig:example_scm}
     \end{subfigure}
     \hfill
     \begin{subfigure}[t]{0.32\textwidth}
         \centering
            \begin{tikzpicture}[node distance =0.95 cm and 0.95 cm]
                \node (A) [label = above right:A, point];
                \node (B) [label = left:b, below left = of A, point];
                \node (C) [label = right:C, point, below right = of A];
                %\node (D) [label = below:D, point, below right = of B];
                
                \node (NA) [label = above:$N_A$, above = of A, yshift=-6mm, xshift=0mm, point];
                %\node (NB) [label = above:$N_B$, above left = of B, yshift=-6mm, xshift=6mm, point];
                \node (NC) [label = above:$N_C$, above = of C, yshift=-6mm, xshift=0mm, point];
                %\node (ND) [label = left:$N_D$, above left = of D, yshift=-6mm, xshift=6mm, point];
                %\node (NB) [label = left:B, below left = of A, point];
                %\node (NC) [label = right:C, point, below right = of A];
                %\node (ND) [label = below:D, point, below right = of B];
        
                %\path (A) edge (B);
                \path (A) edge (C);
                \path (B) edge (C);
                %\path (C) edge (D);
                
                \path[dashed] (NA) edge (A);
                %\path[dashed] (NB) edge (B);
                \path[dashed] (NC) edge (C);
                %\path[dashed] (ND) edge (D);
            \end{tikzpicture}
         \caption{$\mathfrak{C};\DO(B=b)$}
         \label{fig:example_intervention}
     \end{subfigure}
     \hfill
     \begin{subfigure}[t]{0.32\textwidth}
         \centering
            \begin{tikzpicture}[node distance =0.95 cm and 0.95 cm]
                \node (A) [label = above right:A, point];
                \node (B) [label = left:b, below left = of A, point];
                \node (C) [label = right:C, point, below right = of A];
                %\node (D) [label = below:D, point, below right = of B];
                
                \node (NA) [label = above:$N_A|c$, above = of A, yshift=-6mm, xshift=0mm, point];
                %\node (NB) [label = above:$N_B$, above left = of B, yshift=-6mm, xshift=6mm, point];
                \node (NC) [label = above:$N_C|c$, above = of C, yshift=-6mm, xshift=0mm, point];
                %\node (ND) [label = left:$N_D|c$, above left = of D, yshift=-6mm, xshift=6mm, point];
                %\node (NB) [label = left:B, below left = of A, point];
                %\node (NC) [label = right:C, point, below right = of A];
                %\node (ND) [label = below:D, point, below right = of B];
        
                %\path (A) edge (B);
                \path (A) edge (C);
                \path (B) edge (C);
                %\path (C) edge (D);
                
                \path[dashed] (NA) edge (A);
                %\path[dashed] (NB) edge (B);
                \path[dashed] (NC) edge (C);
                %\path[dashed] (ND) edge (D);
            \end{tikzpicture}
         \caption{$\mathfrak{C}|C=c;\DO(B=b)$}
         \label{fig:example_counterfactual}
     \end{subfigure}
     \caption{\Figref{fig:example_scm} depicts an SCM over the variables $\bX = \{A,B,C\}$. Sampling the noise variables $\bN$ and following the structural assignments in topological order gives samples from the observational distribution.
     %$P_{\bX}^{\mathfrak{C}}$. 
     In \Figref{fig:example_intervention}, the intervention $\DO(B=b)$ replaces the structural assignment of $B$ by the hard value $b$. Samples from this modified SCM are samples from the interventional distribution. 
     %$P_{\bX}^{\mathfrak{C};\DO(B=b)}$. 
     \Figref{fig:example_counterfactual} asks \emph{what would have happened had we performed $\DO(B=b)$ given that we actually observed $C=c$?} This counterfactual is obtained by updating the noise distribution and then performing $\DO(B=b)$. Samples from this model are samples from the counterfactual distribution.
     %$P_{\bX}^{\mathfrak{C}|C=c;\DO(B=b)}$. 
     }
     \label{fig:example_ladder}
\end{figure*}
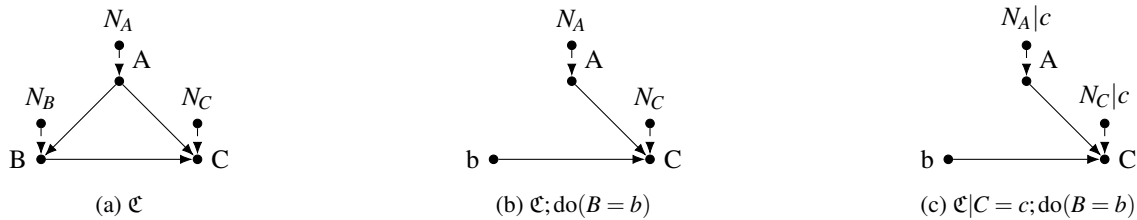

\subsection{Structural Causal Models (SCMs)}
A \emph{structural causal model} $\mathfrak{C}$ is a tuple $(\bX, \bN, \bF, P_{\bN})$, where $P_{\bN}$ is a \emph{noise distribution} over the (exogeneous) noise variables $\bN$ and $\bF = \{f_1, \dots, f_d\}$ is an ordered set of \emph{structural equations} indicating, for each $X_i \in \bX$, how its value is determined by its parents and noise:
\begin{equation}
\label{eq:struct_assign}
    X_i \coloneqq f_i(\bPA_i, N_i),
\end{equation}
%\begin{equation}
%\label{eq:struct_assign}
%X_i \coloneqq \begin{cases}
%N_i, \ & \text{if } \bPA_i = \emptyset, \\
%f_i(\bPA_i, N_i),  \ & \text{otherwise},
%\end{cases}
%\end{equation}
where $\bPA_i \subseteq \bX$, and $N_i$ is the noise variable associated with $X_i$.
The noise models variations due to ignored variables or inherent randomness. 
We assume that the noise variables are independent~\cite{Pearl:2009}.
% , \ie,
% \begin{equation}
%     P_\bN(n_1, \dots, n_d) = \prod\limits_{i=1}^d P_{N_i}(n_i).
% \end{equation}
The associated causal graph $\mathcal{G}$ is obtained by viewing each variable in $\bX$ as a vertex and drawing an arrow from each parent in $\bPA_i$ to $X_i$.
%In  the rest of the paper, we follow this assumption.

\xhdr{Assumptions}
Throughout the paper, we assume that all models satisfy four standard assumptions: the \emph{Markov property}, \emph{causal minimality}, \emph{causal faithfulness}, and \emph{positiveness} \cite{Peters2017}. 
% These are described in details in
% Appendix F.
% \Appref{app:proof-prelim}.

\subsection{Observational, Interventional, and Counterfactual Distributions}
\label{sec:obs-int-count}

\Figref{fig:example_ladder} is a graphical illustration of a three-variable SCM with queries about the observational, interventional, and counterfactual distributions. We define these distributions next.

\xhdr{Observational}
A causal model $\mathfrak{C}$ entails a unique joint distribution of $\bX = (X_1, \dots, X_d)$ called the \emph{observational distribution} and noted $P^{\mathfrak{C}}_{\bX}$~\cite{Peters2017}. To sample from $\mathfrak{C}$, we can simply sample from the noise distribution $P_{\bN}$ and use the structural assignments (\cf\ \Eqnref{eq:struct_assign}) following the topological order of $\bX$ in the causal graph $\mathcal{G}$.

\xhdr{Interventional}
%A causal model also entails interventional and counterfactual distributions.
An intervention on the set of variables $\bI \subset \bX$ of the causal model $\mathfrak{C}$ consists of replacing the structural assignments (\cf\ \Eqnref{eq:struct_assign}) of variables in $\bI$ by forcing them to specific values $\bI=\bi$, a so-called \emph{hard intervention}.\footnote{For simplicity, we focus on hard intervention. However, the approach can easily be extended to soft interventions.}
The new causal model obtained from $\mathfrak{C}$ via the intervention $\bI = \bi$ is denoted by $\mathfrak{C}; \DO(\bI = \bi)$~\cite{Pearl:2009}. Graphically, the \emph{interventional model} is obtained by removing all incoming edges to the nodes in $\bI$. After sampling the noise and following the new structural assignments, we obtain samples from the \emph{interventional distribution} of $\bX$, denoted by $P^{\mathfrak{C}; \DO(\bI = \bi)}_{\bX}$. 
%The interventional distribution is the joint distribution of the modified causal model $\mathfrak{C}; \DO(\bI = \bi)$.

%More generally, we may draw $I$ according to a predefined probability distribution $P_{\bI}$. In this case, we use the notation $P^{\mathfrak{C}; \DO(\bI \sim P_{\bI})}_{\bX}$ for the interventional distribution.

%More generally, the choice of $\bi$ can also follow a predefined distribution $P_I$. For each value
%The new causal model obtained from $\mathfrak{C}$ by performing the intervention $\bI := \bi, \bi \sim P_I$ is noted $\mathfrak{C}; do(I \sim P_I)$~\cite{Peters2017}. It entails a new probability distribution on $\bX$ which we call the \emph{interventional distribution} and note $P^{\mathfrak{C}; do(I \sim P_I)}_X$. 

%The interventional distribution allows us to predict the outcome of the causal model after it has been intervened on by external forces. 
%In practice, this is important for decision making and policy change~\cite{PearlMackenzie18}.

\xhdr{Counterfactual}
At the counterfactual level, we first (partially) observe the causal model in some state $\bE = \be$, where $\bE \subseteq \bX$ is called the \emph{evidence set.} Then we ask: ``Given that $\bE=\be$ actually happened, what would have happened had we done the intervention $\bI = \bi$?'' This is different from the interventional level, where we only ask: ``In general, what happens if we do $\bI = \bi$?'' We now take into account the additional specific information provided by the evidence $\bE = \be$.

Consider the causal model $\mathfrak{C}$ with noise distribution $P_{\bN}$ for which we have some evidence $\bE = \be$. The \emph{counterfactual model} induced by $\mathfrak{C}$ and $\bE=\be$ is denoted by $\mathfrak{C}|\bE=\be$ and is identical to $\mathfrak{C}$ except for the noise distribution $P_{\bN|\bE=\be}$ which has been updated given the evidence using Bayes' rule~\cite{Pearl:2009}:
\begin{equation}
\label{eq:noise bayes}
    P_{\bN|\bE=\be}(\bn) = \frac{P_{\bE|\bN=\bn}(\be)}{P_{\bE}(\be)} P_{\bN}(\bn).
\end{equation}
%The original $P_{\bN}$ serves as the prior, and we compute the posterior distribution $P_{\bN|\bE=\be}$ given the evidence. 
The updated noise variables are not necessarily independent anymore. Note the difference in notation between the induced counterfactual model $\mathfrak{C}|\bE=\be$ and the induced interventional model $\mathfrak{C}; \DO(\bI=\bi)$. The former corresponds to updating the noise distribution, whereas the latter corresponds to modifying the structural assignments of variables $\bI$. 
%The notation $\mathfrak{C}|\bX=\bx$ emphasizes that the \emph{new} causal model is obtained from $\mathfrak{C}$ after it has been observed in some particular state $\bX=\bx$.

%Given some observation $\bX = \bx$, the counterfactual causal model $(\mathfrak{C}|X=x)$ is given by $ (\mathcal{G}, \bF, P_{N|X=x})$ where the noise distribution has been updated given the observation using Bayes formula~\cite{Pearl:2009}. The notation $\mathfrak{C}|X=x$ emphasizes that the \emph{new} causal model is obtained from $\mathfrak{C}$ after it has been observed in some particular state $X=x$. Intuitively, after observing the causal model in state $X=x$, we can infer what were the most likely values of the noise variables that lead to this state $P_{N|X=x}$. We then simply use this new information about the noise variables while keeping the rest of the causal model unchanged. 

A counterfactual query corresponds to an intervention $\DO(\bI=\bi)$ in the counterfactual model $\mathfrak{C}|\bE=\be$. 
Again, this intervention entails a distribution of $\bX$, called the \emph{counterfactual distribution} and denoted by $P^{\mathfrak{C}|\bE=\be; \DO(\bI=\bi)}_{\bX}$. 

\subsection{Metrics, Pseudometrics, and Premetrics}
\label{sec:metrics}
A \emph{metric} $d$ satisfies the four axioms of \emph{non\hyp negativity} ($d(x,y) \ge 0$), \emph{identity of indiscernibles} ($x = y \iff d(x,y) = 0$), \emph{symmetry} ($d(x,y) = d(y,x)$), and the \emph{triangle inequality} ($d(x,z) \leq d(x,y) + d(y,z)$).
% A \emph{pseudometric} relaxes the identity of indiscernibles such that $x=y \implies d(x,y) = 0$, but the implication does not necessarily hold in the opposite direction.
% A \emph{premetric} only satisfies non\hyp negativity and $x=y \implies d(x,y) = 0$.
The causal distances introduced in this paper are pseudometrics (\ie, they relax the identity of indiscernibles), whereas SID (\Secref{intro} and \ref{background}) is a premetric (\ie, only non-negativity and $x=y \implies d(x,y) = 0$ hold).
% We will see that they can be turned into proper metrics in the appropriate spaces.

\section{Related Work}
\label{background}
An important practical application of causal\hyp model distances is the evaluation of causal discovery techniques.
Ideally, one would like to compare an inferred causal model against a given ground-truth model for each type of causal query: observational, interventional, and counterfactual.

The comparison of observational distributions has been studied extensively in machine learning and statistics (\cf\ the overviews by \newcite{theis2015note}
and \newcite{Sriperumbudur:2010:HSE:1756006.1859901}), and distances between distributions have been used to evaluate causal discovery methods, typically by measuring the goodness of fit of the observational distribution induced by a model with respect to empirical samples from the true observational distribution (\cf\ \newcite{singh2017comparative} for an overview).
%These techniques are inherently different from our approach because they do not compare the learned model to the true one and do
Importantly, such methods are inherently limited to the observational level and cannot measure how well the inferred causal model performs at the interventional and counterfactual levels.

%However, in this work, we are interested in comparing two causal models, e.g., comparing a given \emph{gold} causal model to one that has been learned using empirical samples from the gold model. This section discusses existing works which focused on comparing two causal graphs.

%Causal Bayesian Networks can also encode assumptions similar to causal models. In particular, they entails a unique observational distribution and intervention distributions. However, they cannot answer counterfactual queries because the noise variable are not explicit. 
%Bayesian networks have been evaluated mostly based on their goodness of fit in comparison to empirical sample from the \emph{true} observational distribution.

These two levels have received relatively little attention, compared to the observational one.
We are not aware of any previously proposed causal\hyp model distance to consider the counterfactual level, and the few that consider the interventional level focus on a specific aspect of causal models: their causal graphs~\cite{de2009comparison}. For instance, the popular \emph{structural Hamming distance} ($\SHD$) \cite{Acid:2003} counts in how many edges the two input graphs differ.
%\begin{center}
%$\SHD(\mathcal{G},\mathcal{H}) = |\{(i,j) \in V^2| \ G$ and $H$ do not have the same type of edge between $i$  and $j \}|$
%\end{center}
\newcite{SID} argue that previous graph comparison metrics, including SHD, are not in line with the end goal of causal discovery, namely, predicting the outcome of interventions, and propose the \emph{structural intervention distance} ($\SID$), a premetric (\cf\ \Secref{sec:metrics}) that counts the number of pairwise interventional distributions on which two causal models (with graphs $\mathcal{G}$ and $\mathcal{H}$, respectively) disagree:
\begin{eqnarray}
\SID(\mathcal{G},\mathcal{H}) &= |\{(X_i,X_j) \in \bX^2 | \ P(X_i|\DO(X_j)) %\nonumber \\
 \text{ is falsely }\\
& \text{inferred in $\mathcal{H}$ with respect to $\mathcal{G}$}\}|. \label{eqn:SID}
\end{eqnarray}
Under the assumption that the causal models agree on an underlying observational distribution, \newcite{SID} show that the comparison of interventional distributions reduces to a purely graphical criterion.
In particular, when the graphs are the same, both $\SHD$ and $\SID$ are $0$, and $\SHD(\mathcal{G},\mathcal{H}) = 0$ implies $\SID(\mathcal{G},\mathcal{H}) = 0$.
% \cite{SID}.

% \newcite{singh2017comparative} also discuss evaluation methods that measure the performance of an inferred causal model with respect to some predefined causal effect: for fixed $X, Y \in \bX$, is $P^{\mathfrak{C};\DO(X=x)}_{Y}$ estimated correctly?
\newcite{NIPS2018_8155} compare two causal models, but only for the purpose of testing identity. This constitutes a special case of our interventional distance. Recently, \newcite{NIPS2019_9345} argued that causal discovery methods should be evaluated using interventional measures instead of structural ones such as $\SID$ and $\SHD$. The causal distances we introduce are interventional and counterfactual measures.

\subsection{Limitations of Related Work}
\label{ssec:limitations}
%$SHD$ is not a causal measure. In fact, it is possible to find two graphs such that $SID$ is $0$ while $SHD$ attains maximum value~\cite{SID}.
Even though $\SID$ is focused on interventional distributions, it assumes that the underlying observational distribution has been estimated correctly. In practice, this is usually not the case, since the estimation is done using finitely many noisy samples. In general, $\SID$ cannot provide useful answers when the causal models disagree at the observational level. In fact, even when the observational distribution is just slightly off, $\SID$ may still produce highly inaccurate results. 
%It does not approach the true answer when the distance between observational distributions approaches zero (without reaching it).

To illustrate this problem, consider two causal models $\mathfrak{C}_1, \mathfrak{C}_2$, each with two nodes $A,B$. Both models have the graph $A \to B$, with $A \sim \mathcal{N}(0, \sigma_A)$ and $B$'s noise $N_B \sim \mathcal{N}(0, \sigma_B)$:
\begin{align}
\mathfrak{C}_1: \ B &\coloneqq A + N_B, \ P^{\mathfrak{C}_1;\DO(A=a)}_B = \mathcal{N}(a, \sigma_B)\\
\mathfrak{C}_2: \ B &\coloneqq -A + N_B, \  P^{\mathfrak{C}_2;\DO(A=a)}_B = \mathcal{N}(-a, \sigma_B).
\end{align}
The two models predict different values for the intervention $\DO(A=a)$.
% Note that the two models predict different values for the intervention $\DO(A=a)$ for $a \neq 0$:
% \begin{align}
%  &P^{\mathfrak{C}_1;\DO(A=a)}_B = \mathcal{N}(a, \sigma_B), \\
%  &P^{\mathfrak{C}_2;\DO(A=a)}_B = \mathcal{N}(-a, \sigma_B).
% \end{align}
In a toy interpretation, $B$ could be the improvement in life expectancy, and $A$ the daily intake of some drug. Then these two models would give rise to opposite policies given the goal of maximizing life expectancy. This should be reflected by a large distance between the models, but in fact the opposite happens: since $\mathfrak{C}_1$ and $\mathfrak{C}_2$ share the causal graph $\mathcal{G}$, we have $\SHD(\mathcal{G}, \mathcal{G}) = \SID(\mathcal{G},\mathcal{G}) = 0$.

Strictly speaking, $\SID$ cannot even be applied in this case because the observational distributions are not identical.
If, however, $\sigma_A \ll \sigma_B$, the observational distributions become almost indistinguishable, and one might be tempted to apply $\SID$, obtaining a distance of 0 although the interventional distributions still give rise to opposite policies.
We provide more details about this problem and its resolution in Appendix A
% \Appref{app:case_study} 
(\cf\ footnote~\ref{fn:app}).

Another limitation of $\SID$ is that it is binary:
either two pairwise interventional distributions are the same or not (\cf\ \Eqnref{eqn:SID}). It does not quantify the difference. 
In fact, for practical applications, two slightly wrongly inferred interventional distributions might be preferable to one completely wrongly inferred distribution. Also, if one has prior knowledge about which interventions are more critical, one might want to reflect this in the evaluation.
Finally, $\SID$ and $\SHD$ cannot compare causal models at the counterfactual level as they ignore the structural equations and noise distributions (\cf\ \Eqnref{eq:struct_assign}). 
% In contrast, we now propose distances for comparing causal models at all rungs of the ladder of causation.

\section{Definition of Causal Distances}
\label{distances}
Let $\bX$ be a set of random variables, and $\mathfrak{C}_1, \mathfrak{C}_2$ two causal models defined over them. We now introduce natural formulations of distances at the observational, interventional, and counterfactual level. Intuitively, they build upon an underlying distance between probability distributions and mirror the hierarchical aspect of Pearl and Mackenzie's ladder \shortcite{PearlMackenzie18}. 
%The proofs of all theorems stated in this section can be found in \Appref{app:proofs}.

\subsection{Observational Distance (OD)}
Let $P^{\mathfrak{C}_1}_{\bX}, P^{\mathfrak{C}_2}_{\bX}$ be the observational distributions induced by $\mathfrak{C}_1, \mathfrak{C}_2$.
The \emph{observational distance} ($\OD$) is trivial and corresponds to choosing a distance $D$ between probability distributions:
%The intuitive way to define the \emph{observational distance} (OD) is by comparing the observational distributions. Trivially, it corresponds to choosing a distance between probability distributions.
%is to choose a distance $D$ on probability distributions:
\begin{equation}
\label{eqn:OD}
  \OD(\mathfrak{C}_1, \mathfrak{C}_2) = D \left(P^{\mathfrak{C}_1}_{\bX}, P^{\mathfrak{C}_2}_{\bX}\right).
\end{equation}
Example choices for $D$ include the Hellinger, total variation, or Wasserstein distance. 
%Alternatively, one may use divergences such as Kullback--Leibler (at the cost of losing the pseudometric properties, \cf\ \Secref{sec:metrics}).

%or statistical tests distinguishing whether the samples from the two causal models come from the same distribution.

%Then, OD inherits the properties of $D_{\mathbb{P}}$. In particular, if KL divergence is used then OD will not be symmetric.

%In general, we only get samples from $P^{\mathfrak{C}_i}_X,  i \in \{1,2\}$ and we estimate $OD$ by the sample distance.

\subsection{Interventional Distance (ID)}
 An intuitive way to compare two causal models $\mathfrak{C}_1, \mathfrak{C}_2$ at the interventional level is to compare all their interventional distributions.
 Let $I$ denote the node on which the intervention is performed and $\mu$ a distribution over nodes that weighs the interventions on each node. 
 %We touch upon multi-node interventions in \Secref{ssec:mutlinodes}.
 In the absence of such information, $\mu$ may be chosen as the uniform distribution. Then, the \emph{interventional distance} ($\ID$) is defined as
 \begin{align*}
&\ID(\mathfrak{C}_1, \mathfrak{C}_2) =
\mathbb{E}_{I \sim \mu} \mathbb{E}_{i \sim P_I} \left[ \OD(\mathfrak{C}_1; \DO(I=i), \mathfrak{C}_2; \DO(I=i) ) \right].
%\frac{1}{d+1}\sum\limits_{I \in \bX \cup \{\emptyset\}} \mu(I)
\end{align*}

%\begin{align*}
%&\ID(\mathfrak{C}_1, \mathfrak{C}_2)
%=
%\mathbb{E}_{I \sim \mu} \mathbb{E}_{i \sim P_I} \left[ \OD(\mathfrak{C}_1; \DO(I=i), %\mathfrak{C}_2; \DO(I=i) ) \right],
%\end{align*}
 
 By convention, we include the empty intervention $I=\emptyset$, which corresponds to the observational distance $\OD$.
 
 In words, $\ID$ is the expected deviation in the interventional distributions if we sample a node $I$ on which to intervene according to $\mu$ and sample its value according to $P_I$.
 
 The expectation $\mathbb{E}_{i \sim P_I}$ indicates that $I$'s values are drawn from the distribution $P_I$. For instance, $P_I$ can be uniform for discrete models or standard Gaussian for continuous models. We only enforce $P_I$ to be strictly positive for all possible values that $I$ can take.
 
 Variables in $\bX$ may have different value scales, such that the averaged distances may be dominated by a few variables. This can be fixed by co-opting the weights $\mu$ in order to normalize each variable. In the tool we release, we give the option to $z$-normalize each variable and compute a scale-invariant $\ID$.
 
 Note that $\mu$ can give weights of $0$ to some nodes, \eg, if they are unobservable to the causal discovery method. 

\subsection{Counterfactual Distance (CD)}
The natural way to compare models at the counterfactual level is to consider their interventional distance on all counterfactual models, i.e., the counterfactual models induced by all possible evidences. Let $E = e$ denote the observed evidence and $\nu$ a distribution over nodes that weighs the counterfactuals induced by observing each node. Similar to $\mu$, in the absence of further information, $\nu$ may be chosen to be uniform. 
%We also discuss multi-node evidence sets in \Secref{ssec:mutlinodes}.
%
%This would mean $2^d$ different evidence sets. However, we prove the counterfactual analog to \Thmref{th:single_node_intervention} (\cf\ \Appref{app:proofs}):
%
%\begin{restatable}{theorem}{singlecounter}
%\label{th:single-counter}
%    Two causal models agree on the outcome of every counterfactual query if and only if they agree on the outcome of every counterfactual query induced by a single-node evidence set.
%\end{restatable}
%
%Two causal models can be interventionally equivalent but still produce different answers for counterfactual queries~\cite{Pearl:2009}.
%Let $\omega_{\bX}$ be the set of all possible observations $\bX=\bx$. Then, $\mathfrak{C}_1|\bX = \bx$ and $\mathfrak{C}_2|\bX = \bx$ are the counterfactual models associated to $\mathfrak{C}_1$ and $\mathfrak{C}_2$ after observing $\bX=\bx$. 
Then, the \emph{counterfactual distance} ($\CD$) is defined as
\begin{align*}
&\CD(\mathfrak{C}_1, \mathfrak{C}_2) = 
\mathbb{E}_{E \sim \nu} \mathbb{E}_{e \sim P_{E}} [\ID(\mathfrak{C}_1|E=e, \mathfrak{C}_2|E=e)],
\end{align*}

%\begin{align*}
%\CD(\mathfrak{C}_1, \mathfrak{C}_2)
%=
%\mathbb{E}_{E \sim \nu} \mathbb{E}_{e \sim P_{E}} [\ID(\mathfrak{C}_1|E=e, \mathfrak{C}_2|E=e)],
%\end{align*}
%
%\begin{align}
%&CD(\mathfrak{C}_1, \mathfrak{C}_2) = \mathbb{E}_{\nu} \mathbb{E}_{e \sim \mathcal{U}(\Omega_{E})} [ID(\mathfrak{C}_1|E=e, \mathfrak{C}_2|E=e)] \\
%&= \sum\limits_{E \in \bX} \nu(E) \mathbb{E}_{e \sim \mathcal{U}(\Omega_{E})} [ID(\mathfrak{C}_1|E=e, \mathfrak{C}_2|E=e)] 
%\end{align}
By convention, we include the empty evidence $E=\emptyset$, which corresponds to the interventional distance $\ID$.

%where $\nu$ is a distribution over nodes that indicates how much weight should be put on the counterfactual model induced by observing node $E$. 
The expectation $\mathbb{E}_{e \sim P_E}$ indicates that $E$'s values are drawn from the distribution $P_E$. In the absence of information, $P_E$ may be uniform for discrete models or standard Gaussian for continuous models.

%By convention, we include the empty evidence $E=\emptyset$ whose interventional distance is $\ID$.

%Similar to $\mu$, it is uniform when no information is provided. The double expectation $\mathbb{E}_{\nu} \mathbb{E}_{e \sim U(\Omega_{E})}$ indicates that $CD$ is the averaged $ID$ over all counterfactual models induced by a single-node observation $E$ when the observed value $e$ is sampled uniformly from $\Omega_E$, the support of $E$. 

%\begin{align}
%&CD(\mathfrak{C}_1, \mathfrak{C}_2) = E_{U(\Omega_{\bX})}\left[ ID(\mathfrak{C}_1|\bE=\be, \mathfrak{C}_2|\bX=\bx) \right] \\
%& = E_{U(\Omega_{\bX})}\left[ E_{\mu} \left[D[ P^{\mathfrak{C}_1|\bX=\bx; do(\bI)}_{\bX}, P^{\mathfrak{C}_2|\bX=\bx; do(\bI)}_X ] \right] \right]
%
%\end{align}
%The observations $\bx$ are sampled uniformly on the support $\Omega_X$ of $\bX$. Then, $CD$ corresponds to computing the expected intervention distance $ID$ over all counterfactual causal models. Each counterfactual $\mathfrak{C}_1|\bX=\bx$ model is induced by an observation $\bX=\bx$ and $\mathfrak{C}_1|\bX=\bx$ is obtained from $\mathfrak{C}_1$ by updating the noise distribution using Bayes rule.

\section{Basic Properties of Causal Distances}
\label{properties}
In this section, we assume that both $\mu$ and $\nu$ are uniform over the set of nodes. The proofs are given in Appendix F.
% \Appref{app:proofs}.

Each distance builds on top of the distance defined at the level below (CD on ID; ID on OD), thus reflecting the hierarchical structure of the ladder of causation. Furthermore, one can verify the following connections.
\begin{restatable}{theorem}{cdladder}
\label{th:cd_ladder}
For two causal models $\mathfrak{C}_1$ and $\mathfrak{C}_2$ over the variables $\bX$, we have:
\begin{align}
    &\ID(\mathfrak{C}_1, \mathfrak{C}_2) \leq (|\bX|+1)\CD(\mathfrak{C}_1, \mathfrak{C}_2) \\
    &\OD(\mathfrak{C}_1,\mathfrak{C}_2) \leq (|\bX|+1)\ID(\mathfrak{C}_1, \mathfrak{C}_2)
\end{align}
\end{restatable}
In particular, counterfactual equivalence implies interventional equivalence, which in turn implies observational equivalence.
% (corresponding to the case $\epsilon = 0$).

\subsection{Connection with Graph-based Metrics}
The interventional distance ($\ID$) is related to the graph-based $\SID$ and $\SHD$ via
\begin{restatable}{theorem}{idsid}
\label{th:id_sid}
  For two causal models $\mathfrak{C}_1, \mathfrak{C}_2$ with causal graphs $\mathcal{G}_1, \mathcal{G}_2$,
    \begin{align}
        \label{eeq:implication_1}
       &\ID(\mathfrak{C}_1, \mathfrak{C}_2) = 0 \implies \SHD(\mathcal{G}_1, \mathcal{G}_2) = 0 \implies \SID(\mathcal{G}_1, \mathcal{G}_2) = 0.
    \end{align}
  The reverse directions do not hold in general.
\end{restatable}

A further connection between $\SID$ and our causal distances is given by
\begin{restatable}{theorem}{sididod}
\label{th:id_sid_od}
  Let $\mathfrak{C}_1, \mathfrak{C}_2$ be two causal models with causal graphs $\mathcal{G}_1, \mathcal{G}_2$. When $\OD(\mathfrak{C}_1, \mathfrak{C}_2) = 0$,
    \begin{align}
     \SID(\mathcal{G}_1, \mathcal{G}_2) = 0 \iff \ID(\mathfrak{C}_1, \mathfrak{C}_2) = 0.
    \end{align}
    When $\OD(\mathfrak{C}_1, \mathfrak{C}_2) \neq 0$ the equivalence does not hold.
\end{restatable}
From \Thmref{th:id_sid}, we know that $\ID$ being $0$ guarantees that $\SID$ is $0$. But $\SID$ being $0$ only ensures that $\ID$ is $0$ in the specific case where $\OD$ is also $0$.

\subsection{Hidden Variables}
Until now, we considered the comparison of two Markovian causal models, i.e., with no hidden confounders. 
We may ask what happens in the non-Markovian case, where one or both models have hidden confounders. 

If both models have hidden confounders that can be intervened on, we cannot bound the expected difference between two models, as the outcome of intervening on the hidden confounder can be made arbitrarily large, as shown by \Figref{fig:hidden}.

\begin{figure}
\centering
\begin{tikzpicture}[node distance =0.85 cm and 0.85 cm]
    \node (Z) [label = above:$Z$, point];
    \node (X) [label = below:$X$, below left = of Z, point];
    \node (Y) [label = below:$Y$, point, below right = of Z];
        
    \path[dashed] (Z) edge node[above, el] {$\lambda$} (X);
    \path[dashed] (Z) edge node[above, el] {$-\lambda$} (Y);
\end{tikzpicture}
\qquad
\qquad % <----------------- SPACE BETWEEN PICTURES
\begin{tikzpicture}[node distance =0.85 cm and 0.85 cm]
    \node (Z) [label = above:$Z$, point];
    \node (X) [label = below:$X$, below left = of Z, point];
    \node (Y) [label = below:$Y$, point, below right = of Z];
        
    \path (Z)[dashed] edge node[above, el] {$-\lambda$} (X);
    \path (Z)[dashed] edge node[above, el] {$\lambda$} (Y);
\end{tikzpicture}
\caption{$Z \sim \mathcal{N}(0,1)$ is a hidden confounder in both graphs. The edges indicate a multiplicative factor, e.g., $X = \lambda Z$ in the left graph. The two models have the same joint distribution on $(X,Y)$ and the same graph. Yet, $\DO(Z=z)$ for $z \neq 0$ results in two different joint distributions. Their (Wasserstein) distance can be made arbitrarily large by increasing $\lambda$.} \label{fig:hidden}
\end{figure}
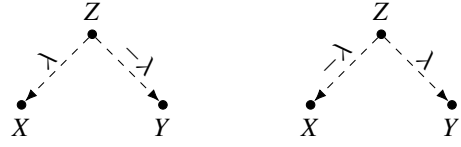

However, this constitutes a fairly peculiar scenario. Indeed, it is expected that comparing ``incomplete'' models can only give partial information.
% \footnote{
It is similar to trying to establish a distance between vectors where one dimension remains hidden.
% } 
In practice, if we wish to compare two causal models either (i) one is fully known (e.g., the gold standard model to which we compare a model inferred by a causal discovery technique with variables unobservable during training) or (ii) the hidden variables cannot be intervened on. In this latter case, $\OD$ and $\ID$ computed on the observed subset preserve their interpretation.  

%(ii) However, as long as no hidden confounder is intervened on, it remains possible to compute $\OD$ and $\ID$ on the observed set of nodes while preserving their interpretation on this subset. In particular, when one model is fully known (e.g., gold standard) while the other may have hidden confounders or simply a smaller subset of observed nodes (e.g., the outcome of a causal discovery algorithm trained with hidden variables). Then, $\OD$ and $\ID$ computed on the observed subset of nodes provides the evaluation of the partial causal model.

%In general, counterfactuals cannot be computed with semi-Markovian models as we cannot update the noise corresponding to hidden nodes under the evidence. In such case, $\CD$ is not defined.

\subsection{Practical Implementation} 
In Appendix D, we describe the practical details related to the implementation and efficient estimation of the causal distances. We also discuss how to handle both continuous and discrete variables. This implementation results in a tool that we make publicly available to the community.

% \subsection{Pseudo-metrics}
% % \vspace{4pt}
% Technically, $\OD$, $\ID$, and $\CD$ are not metrics, since the identity of indiscernibles does not hold: there can be two distinct causal models with an $\OD$, $\ID$, or $\CD$ of $0$. If $D$ satisfies the other metric axioms (but not otherwise, e.g., if $D$ is the asymmetric KL\hyp divergence; \Secref{sec:metrics}), our distances are pseudometrics.

% Like all pseudometrics, however, they can be turned into proper metrics by considering equivalence classes as the objects of comparison, where the equivalence class of a model is the set of all models to which it has distance $0$. Interestingly, these equivalence classes are tightly connected to problems of identifiability \cite{Pearl:2009}.

\section{Experiments}
\label{experiments}
% We now conduct a wide range of experiments.
We now conduct experiments with the causal distances, on both synthetic and real-world causal models.
In all experiments, we use the sample Wasserstein distance~\cite{villani2008optimal} as the underlying distance $D$ between probability distributions (\cf\ \Eqnref{eqn:OD}).
%, as it does not require binning and has good convergence properties
In Appendix B, 
% \Appref{app:efficiency}
 we validate the sample efficiency and sensitivity of the causal distances.

\begin{figure}
    \centering
    % \captionsetup{justification=centering}
    % \begin{subfigure}[t]{0.31\columnwidth}
    %      \centering
    %     %  \captionsetup{justification=centering}
    %      \includegraphics[width=\textwidth]{images/geometry_sid.pdf}
    %      \caption{MDS with $\SID$.}
    %      \label{fig:geom_sid}
    %  \end{subfigure}
    %  \hfill
     \begin{subfigure}[t]{0.47\columnwidth}
         \centering
        %  \captionsetup{justification=centering}
         \includegraphics[width=\textwidth]{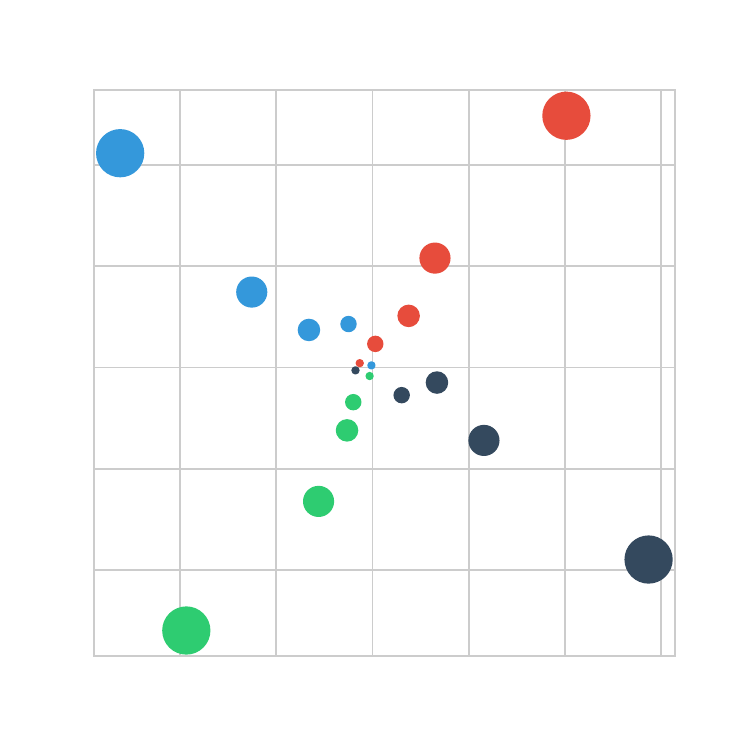}
         \caption{MDS with $\ID$.}
         \label{fig:geom_id}
     \end{subfigure}
     \hfill
    \begin{subfigure}[t]{0.47\columnwidth}
         \centering
        %  \captionsetup{justification=centering}
         \includegraphics[width=\textwidth]{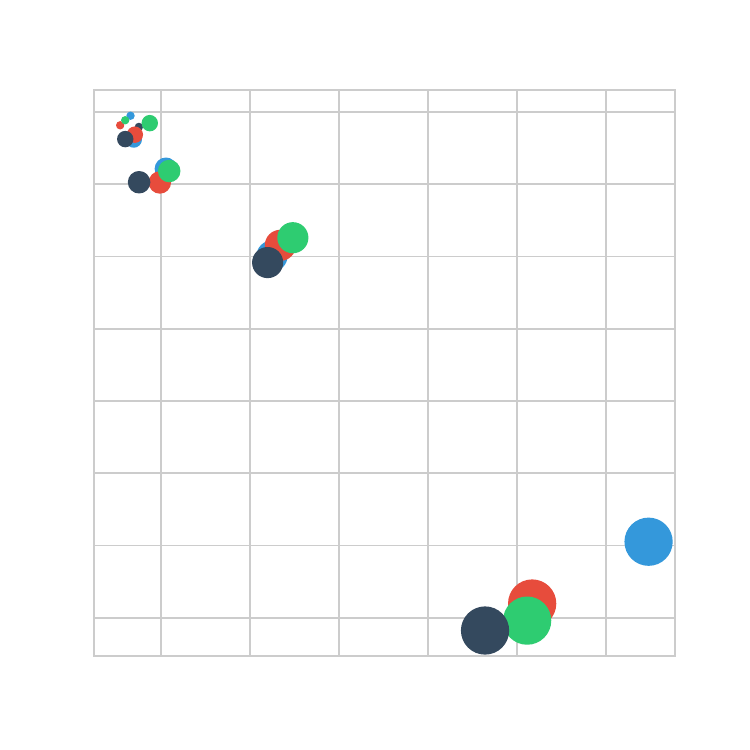}
         \caption{MDS with $\OD$.}
         \label{fig:geom_od}
     \end{subfigure}
    \caption{Comparison of multidimensional scaling (MDS) embeddings induced by $\OD$ and $\ID$ using the four causal models described in \Secref{sec:geometry} (one color per model; size represents the strength of $\beta$).}
    %With $\SID$, all models would collapse on one of two points, one for the graph $A \to B$ and one for the graph $B \to A$.}
    %\label{fig:sid-id-od}
\end{figure}

\subsection{Geometry of Causal Distances}
\label{sec:geometry}
First, we illustrate the intuitive geometry induced by our causal distances. In particular, we consider simple models with two nodes $A$ and $B$ using linear structural equations and Gaussian noise. We let $\beta > 0$ denote the strength of the causal connection, let $N \sim \mathcal{N}(0,1)$ be $B$'s noise, and consider four types of models:
\begin{eqnarray*}
    % A \nearrow B&:& 
    % A [\nearrow, \searrow] B&:
    & A \sim \mathcal{N}(0,1) \text{ and } B \coloneqq \pm \beta A + N,\\
    % A \searrow B&:& A \sim \mathcal{N}(0,1) \text{ and } B \coloneqq -\beta A + N,\\
    % B [\nearrow, \searrow] A&:
    & B \sim \mathcal{N}(0,1) \text{ and } A \coloneqq \pm \beta B + N,
    % B \searrow A&:& B \sim \mathcal{N}(0,1) \text{ and } A \coloneqq -\beta B + N.
\end{eqnarray*}
We construct five models of each type with $\beta =0.1, 0.5, 1, 2, 5$, respectively, resulting in $20$ models overall.
We then compute the pairwise distances between all models using $\ID$ and apply multidimensional scaling to obtain 2D embeddings of all models. The result is depicted in \Figref{fig:geom_id} and exhibits the geometrical structure induced by $\ID$, where each type of model creates its own branch, and larger values of $\beta$ push the different types further apart. When $\beta \to 0$, all models converge to a model where $A$ and $B$ are causally disconnected. When viewed in 3D, equal angles would separate all pairs of branches.
% \footnote{Visualization will be available in the accompanying Jupyter notebook.}

In contrast, $\SID$ %depicted in \Figref{fig:geom_sid} 
induces a much poorer geometry where each model is projected on one of two points: one representing the graph $A \to B$, the other, the graph $B \to A$.
With $\OD$, shown in \Figref{fig:geom_od}, the models form one branch in the 2D embedding. They are only distinguished based on the amplitude of $\beta$; neither the sign nor the orientation of the graph are captured.

\begin{figure}
    \centering
    \begin{subfigure}[t]{0.48\columnwidth}
         \centering
        %  \captionsetup{justification=centering}
         \includegraphics[width=\textwidth]{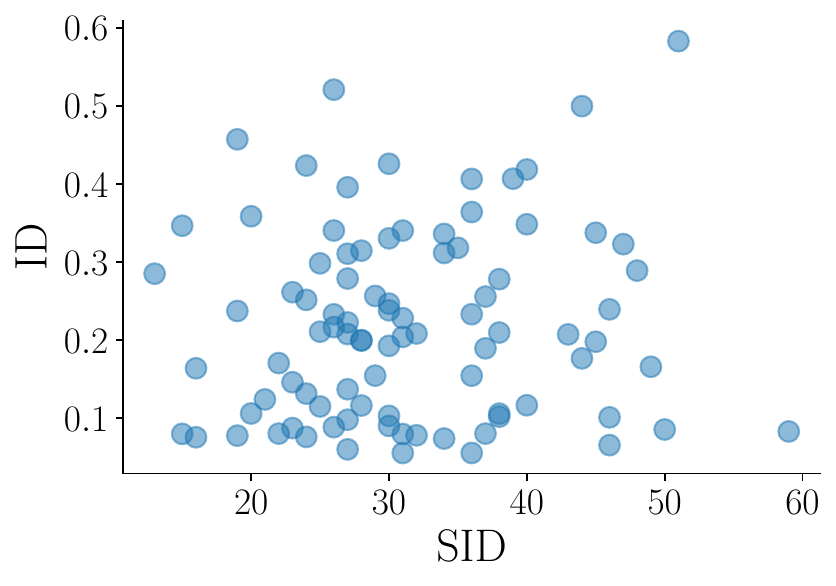}
         \caption{$\SID$ vs. $\ID$.}
         \label{fig:sid_id}
     \end{subfigure}
     \hfill
    %  \begin{subfigure}[t]{0.23\columnwidth}
    %      \centering
    %     %  \captionsetup{justification=centering}
    %      \includegraphics[width=\textwidth]{images/shd_od.pdf}
    %      \caption{Scatter-plot comparing $\SHD$ and $\OD$.}
    %      \label{fig:shd_od}
    %  \end{subfigure}
    %  \hfill
    \begin{subfigure}[t]{0.48\columnwidth}
         \centering
        %  \captionsetup{justification=centering}
         \includegraphics[width=\textwidth]{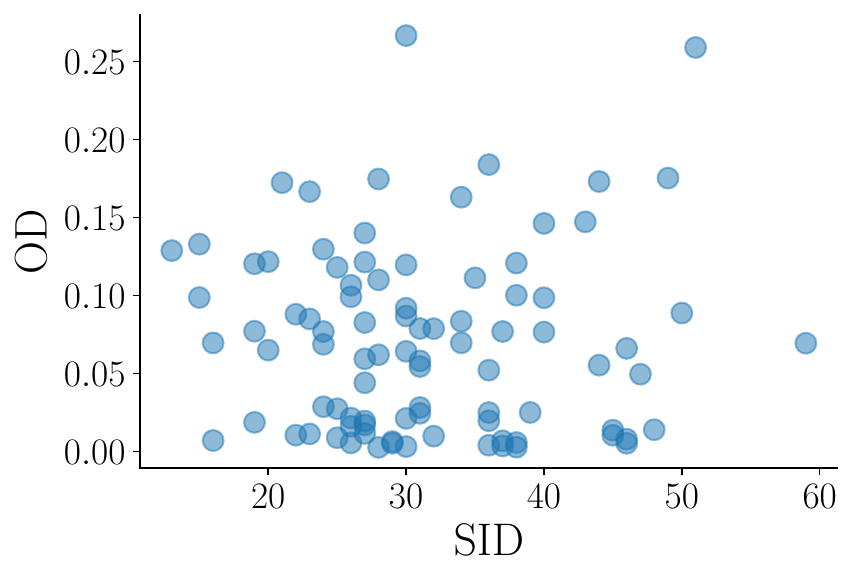}
         \caption{$\SID$ vs. $\OD$.}
         \label{fig:sid_od}
     \end{subfigure}
    %  \hfill
    %  \begin{subfigure}[t]{0.23\columnwidth}
    %      \centering
    %     %  \captionsetup{justification=centering}
    %      \includegraphics[width=\textwidth]{images/shd_id.pdf}
    %      \caption{Scatter-plot comparing $\SHD$ and $\OD$.}
    %      \label{fig:shd_id}
    %  \end{subfigure}
    %\includegraphics[width=0.68\linewidth]{images/OD_ID_SID.pdf}
    \caption{Comparisons between ${\ID}$ and ${\OD}$ against $\SID$ on 90 randomly sampled pairs of causal models.}
    %\label{fig:sid-id-od}
\end{figure}

\subsection{Comparison of Causal Distances and SID}
\label{sec:comparison}
%There are several different metrics that one can derive to compare causal models depending of which property one aim to measure. We saw $SID$ as a particularly relevant example concerning causal discovery. In this experiment, we show that $SID$ and SID indeed produce different results and can be used in complement. While SID focuses on the graph structure, $ID$ focuses on the actual distances between the intervention distributions.

Whereas \Thmref{th:id_sid_od} connects $\ID$, $\OD$, and $\SID$ when two of these quantities are $0$, we now empirically investigate their relationship when they deviate from $0$.
\Figref{fig:sid_id} shows a scatter plot comparing  ${\ID}$ and $\SID$, where each dot is a pair of random causal models. 
% from the parametrizations defined above (linGauss, linNGauss, GPAddit, GP, Discrete).
%Except when $\tilde{ID}\approx0$ and $SID\approx0$ (see theorem \ref{th:id_sid}), 
As we see, there is little correlation between ${\ID}$ and $\SID$. It is possible to find pairs of causal models with low ${\ID}$ but high $\SID$, and \textit{vice versa}.
% The same behaviour is observed when $\SID$ is replaced by $\SHD$, as depicted by \Figref{fig:shd_od} and \Figref{fig:shd_id}. 
% The observed correlation coefficients are not significant.
These results highlight how the different distances capture different aspects of the models being compared. 
% They should be considered complementary to one another.
%This indicates that outside of the precise case where the quantities are $0$, there is little correlation between the causal distance and $\SID$ 

%It is also interesting to observe some positive correlation between ${\OD}$ and ${\ID}$, which indicates that, given the causal assumptions (encoded by the causal model), there are connections between the different levels of the ladder of causation~\cite{Pearl:2019}. 

%For example, if the causal graph is known, then every intervention distribution can be identified, i.e., computed from the observational distribution.

%These results highlight how the different distances capture different aspects of the models being compared. They should be considered complementary to one another.

\subsection{Evaluation of Causal Discovery Systems}
\label{sec:discovery eval}

\begin{figure*}
    \centering
    \includegraphics[width=0.99\textwidth]{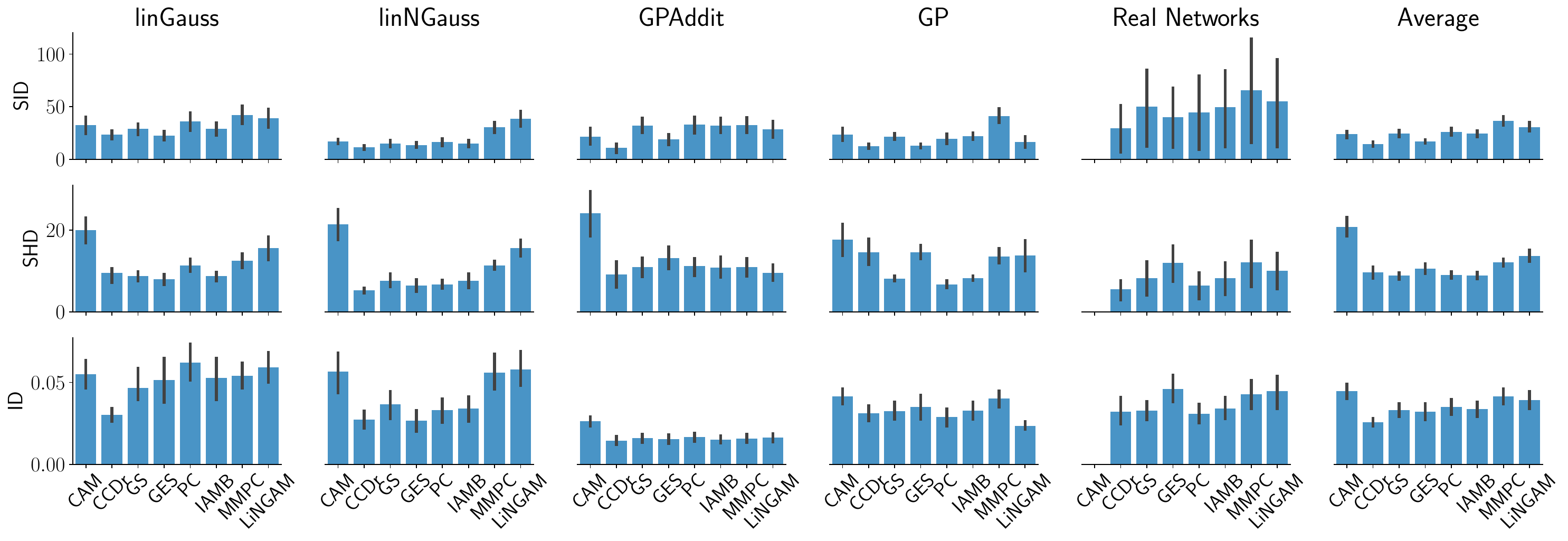}
    \caption{Evaluation of causal discovery techniques on synthetic networks and real-world networks (``Real Networks''). In the first row, models are evaluated by $\SID$, in the second row, by $\SHD$, and in the last row, by $\ID$. The rightmost column shows average performance. Note that CAM yielded errors on some of the real-world networks and is thus not reported. Lower is better. Error bars are 95\% confidence intervals.}
    \label{fig:eval}
\end{figure*}

An important application of causal distances is the evaluation of causal discovery systems. 
% ${\OD}$, ${\ID}$, and ${\CD}$ can precisely evaluate the inferred causal model $\mathfrak{C}$ in comparison to a ground-truth causal model $\mathfrak{D}$ at each rung of the ladder of causation.
In this section, we illustrate this by evaluating several causal discovery systems using both real-world and synthetic causal models.

We considered the following real-world Bayesian causal models from \newcite{bnrepository}:
\textit{Cancer1},
%\cite{Lauritzen:1990},
% , a model of lung cancer (8 nodes, 8 edges);
\textit{Cancer2},
%\cite{korb},
% , a toy model connecting pollution and smoking to lung cancer (5 nodes, 4 edges);
\textit{Earthquake}, 
%\cite{korb}, 
% a model of alarm triggering (5 nodes, 4 edges);
\textit{Survey},
%\cite{scutari2014bayesian}, 
% a model of survey outcomes (6 nodes, 6 edges);
\textit{Protein}, 
%\cite{sachs2005}, 
% a model of protein signaling (11 nodes, 17 edges);
\textit{Child}, and 
% a model of the diseases that birth asphyxia may cause (20 nodes, 25 edges);
\textit{Insurance}.
%\cite{insurance}.
% a model of car insurance policies (27 nodes, 52 edges).
They range from $5$ to $27$ nodes and from $4$ to $52$ edges. 
For synthetic models, we sample random DAGs and random parametrizations using the CDT tool \cite{kalainathan2019causal}.
% \footnote{\url{https://github.com/FenTechSolutions/CausalDiscoveryToolbox/}} 
We consider the following parametrizations: \textit{linear Gaussian (linGauss)}, \textit{linear non-Gaussian (linNGauss)}, \textit{Gaussian process with additive noise (GPAddit)}, and \textit{Gaussian Process (GP)}.

%The main properties of each model are detailed in Table \ref{tab:detail-scm}.
For each model, we sample 2,000 observations from which causal discovery methods should recover the causal model. Since they are causal Bayesian networks, not structural causal models, we cannot compute counterfactuals~\cite{Pearl:2019}. Thus, we restrict ourselves to ${\ID}$, $\SHD$, and $\SID$.

\xhdr{Systems}
We consider multiple causal discovery methods for recovering the causal graph: CCDr~\cite{Aragam:2015}, PC~\cite{spirtes2000causation}, GES~\cite{chickering2002finding}, GIES~\cite{chickering2002optimal}, MMPC~\cite{mmpc}, IAMB~\cite{IAMB}, LiNGAM~\cite{shimizu2006linear}, and CAM~\cite{spirtes2000causation}.
Some of these techniques only output a partial DAG with undirected edges, and some only output a graph without parameters. 
To obtain a fair comparison of the full Bayesian networks, we perform all parameter estimation via maximum likelihood estimates based on the training data. When only a partial DAG is returned, we use the edge orientation that provides the best goodness of fit after the parameters have been estimated. Alternatively, one could report the mean performance over all possible directed graphs.
% This can be passed as an option to the tool we release.

\xhdr{Remark}
Our distances compare full causal models, \ie, causal graphs after the parameters have been estimated. Here, by fixing the parameter estimation, we measure the impact of the causal graph on the intervention predictions. It also ensures that two methods that output the same graph (DAG or partial DAG) will obtain the same evaluation results. 
Yet, as shown below, it does not mean that these evaluation results agree with graph-based metrics such as $\SHD$ or $\SID$.
% In future work, our distances can also be used to evaluate different parameter estimation techniques by keeping causal graphs fixed, or alternatively, evaluate full causal discovery pipelines.
%To obtain full Bayesian networks from the inferred causal graphs, we combine them with both MLE and MAP estimation of the parameters. Some of these techniques do not always construct a DAG and leave some edges undirected. When this happens, we use the edge orientation which provides the best goodness of fit after the parameters have been estimated.
%Furthermore, we evaluate techniques that directly build a causal model: LiNGAM~\cite{shimizu2006linear}, CAM~\cite{spirtes2000causation}, and two modern techniques based on neural networks: CGNN~\cite{Goudet2018} and SAM~\cite{kalainathan2018sam}.
For the systems, we use the implementations available in CDT. We compute maximum likelihood estimates using the \emph{Pomegranate} python framework.
% \footnote{\url{https://pomegranate.readthedocs.io}}

%In general, every causal graph can be compared using our causal distances after the parameters have been estimated following the same procedure. This measures the impact of the causal graph on the intervention prediction under a fixed parameter estimation. Alternatively, the results can be averaged over several different parameter estimations procedures. By default, the tool we release automatically compare causal graphs based on a fixed parameter estimation.

% For details about how these systems work, we redirect to the original papers.

\xhdr{Results}
In \Figref{fig:eval}, we report the evaluation results broken down by model parametrization. The \emph{Real Networks} block corresponds to the performance of systems averaged over the real-world causal models described above. Results for each causal model are available in Appendix C. 
% \Appref{app:model_perf}.
% Thus, we also perform an evaluation of causal discovery system broken down by model parametrization, shown in \Figref{fig:eval}. The \emph{Real Networks} block corresponds to results of \Tabref{tab:eval-systems} averaged across networks for comparison.
% The results on real networks are reported in \Tabref{tab:eval-systems}. 
We see that different metrics yield different rankings of systems. Thus, the differences between metrics observed in \Secref{sec:comparison} are also consequential in the task of causal discovery. 
% In particular, we observe low agreement between $\SID$ and $\ID$ on the Earthquake and Insurance networks. Also, IAMB and MMPC have the same $\SID$ (16) and $\SHD$ (7) on Cancer2 but different graphs which is distinguished by $\ID$. On the contrary, on the Protein dataset, GS and IAMB have the same graph and, with fixed parameter estimation, the same $\SHD$, $\SID$ and $\ID$. 
These observations emphasize the importance of employing the evaluation metric that captures the desired behavior. If only the observational distribution matters, ${\OD}$ should be used, but then causal discovery may not be needed in the first place. If we care about the expected errors in predicting the outcome of interventions, ${\ID}$ should be used, and $\SID$ can be employed when we focus on the causal graph under the assumption that the underlying observational distribution has already been correctly estimated.
% A peculiarity of $\SID$ is that it outputs integers only, which can result in ties, whereas ${\ID}$ produces continuous values and does not have this problem. 
% Furthermore, $\ID$ is normalized by default whereas $\SID$ and $\SHD$ greatly vary depending on the number of nodes.
Additionally, even a single metric produces different rankings of systems in different scenarios. Causal discovery requires assumptions about the underlying structure of the true causal model, and few guarantees are given when the respective assumptions are not met. Different networks fulfill different assumptions and are best handled by different causal discovery methods.
% This confirms the importance of prior knowledge and causal assumptions.
An evaluation using causal distances such as $\OD$, $\ID$, and $\CD$ is indispensable for illuminating which causal discovery method is best suited for which kind of data.
Interestingly, $\ID$ clearly reveals that systems struggle most for the linear Gaussian case, which is known to be unidentifiable. While all other cases are identifiable,
Gaussian processes with additive noise (GPAddit)
% the non-linear additive case
seem to be the easiest for existing causal discovery systems. Overall, CCDr seems to perform fairly well in comparison to other systems. 
%This conclusion is supported by both $\SID$ and $\ID$.
In Appendix E, 
% \Appref{app:more_data}
 we show that, contrary to machine learning, causal discovery systems do not benefit from more training data.

% \section{Applications and future work}
% \label{future_work}
% \input{060future_work}

\section{Conclusion}
\label{conclusion}
This paper introduces observational ($\OD$), interventional ($\ID$), and counterfactual ($\CD$) distances between causal models, one for each rung of the ladder of causation (\cf\ \Secref{intro}). Each distance is defined based on the lower-level ones, reflecting the hierarchical structure of the ladder. We study the properties of our distances and propose practical approximations that are useful for evaluating causal discovery techniques. We release a Python implementation of our causal distances.
%\footnote{\url{available-after-reviewing.}}

Our causal distances do not require the unrealistic assumptions of infinite samples and perfect statistical estimation that are currently common in the study of causality~\cite{Pearl:2009}. Also, they quantify the difference between causal models on a continuous, rather than integer, scale and make use of the data at a finer granularity than the usual binary measurements of methods such as SHD and $\SID$ (\cf\ \Secref{background}).

The proposed causal distances have both theoretical and empirical applications, 
%in artificial intelligence and beyond, 
and we hope the research community will use them to advance the study of causality.

\section*{Acknowledgments}
With support from
Swiss National Science Foundation (grant 200021\_185043),
European Union (TAILOR, grant 952215),
and gifts from
Google, Facebook, Microsoft
% }

\clearpage

\bibliographystyle{named}
\bibliography{CD}

\clearpage

\appendix
\section{Case study}
\label{app:case_study}
To illustrate the problems of graph-based metrics like $\SID$, consider the two causal models $\mathfrak{C}_1$ and $\mathfrak{C}_2$ over the two nodes $A$ and $B$ presented in \Secref{ssec:limitations}. Both models have the graph $A \to B$, with $A \sim \mathcal{N}(0, \sigma_A)$ and $B$'s noise $N_B \sim \mathcal{N}(0, \sigma_B)$:
\begin{align}
\mathfrak{C}_1: \ B &= A + N_B, \\
\mathfrak{C}_2: \ B &= -A + N_B.
\end{align}

Even though they have the same causal graph, they predict different values for the intervention $\DO(A=a), a \neq 0$:
\begin{align}
&P^{\mathfrak{C}_1;\DO(A=a)}_B = \mathcal{N}(a, \sigma_B), \\
&P^{\mathfrak{C}_2;\DO(A=a)}_B = \mathcal{N}(-a, \sigma_B).
\end{align}
In a toy interpretation, $B$ could be the improvement in life expectancy, and $A$ the daily intake of some drug. Then these two models would give rise to opposite policies given the goal of maximizing life expectancy.

$\SID$ fails because it relies on the assumption that both models agree on the underlying observational distribution. This is not the case here. Yet, the interventional distributions can indicate opposite policies even when the two observational distributions are made arbitrarily close to each other.

In particular, when $\sigma_A \ll \sigma_B$, the observational distributions become almost indistinguishable. One could expect $\SID$ to converge to the true answer when the observational distributions converge to each other, but this is not the case. Indeed, here, no matter how small is the difference between the two distributions, as long as it is not precisely $0$, $\SID$ will keep predicting $0$ as the number of wrongly inferred interventional distributions. This is particularly problematic because, in practice, we are never in the infinite sample regime and can never estimate perfectly the observational distribution.

\Figref{fig:case_study_1} depicts the observational and interventional distributions of both models under the action $\DO(A=3)$ with $\sigma_A = \sigma_B = 1$.
Similarly, ~\Figref{fig:case_study_2} shows the same distributions, but with $\sigma_A = 0.1 \cdot \sigma_B$.
We observe that the interventional distributions remain unchanged and different from each other ($\ID$ is constant) even if the observational distributions become almost indistinguishable as $\frac{\sigma_A}{\sigma_B}$ becomes smaller.

This problem is resolved by considering the interventional distance $\ID$ instead of $\SID$, since ID recognizes the interventional distributions as strictly different: $\ID(\mathfrak{C}_1, \mathfrak{C}_2) \approx 0.343$ when $\sigma_A = 0.1 \cdot \sigma_B$ (it is $\approx 0.625$ when $\sigma_A = \sigma_B=1$).

\begin{figure*}[t!]
     \centering
     \begin{subfigure}[t]{0.48\textwidth}
         \centering
         %\captionsetup{justification=centering}
         \includegraphics[width=\textwidth]{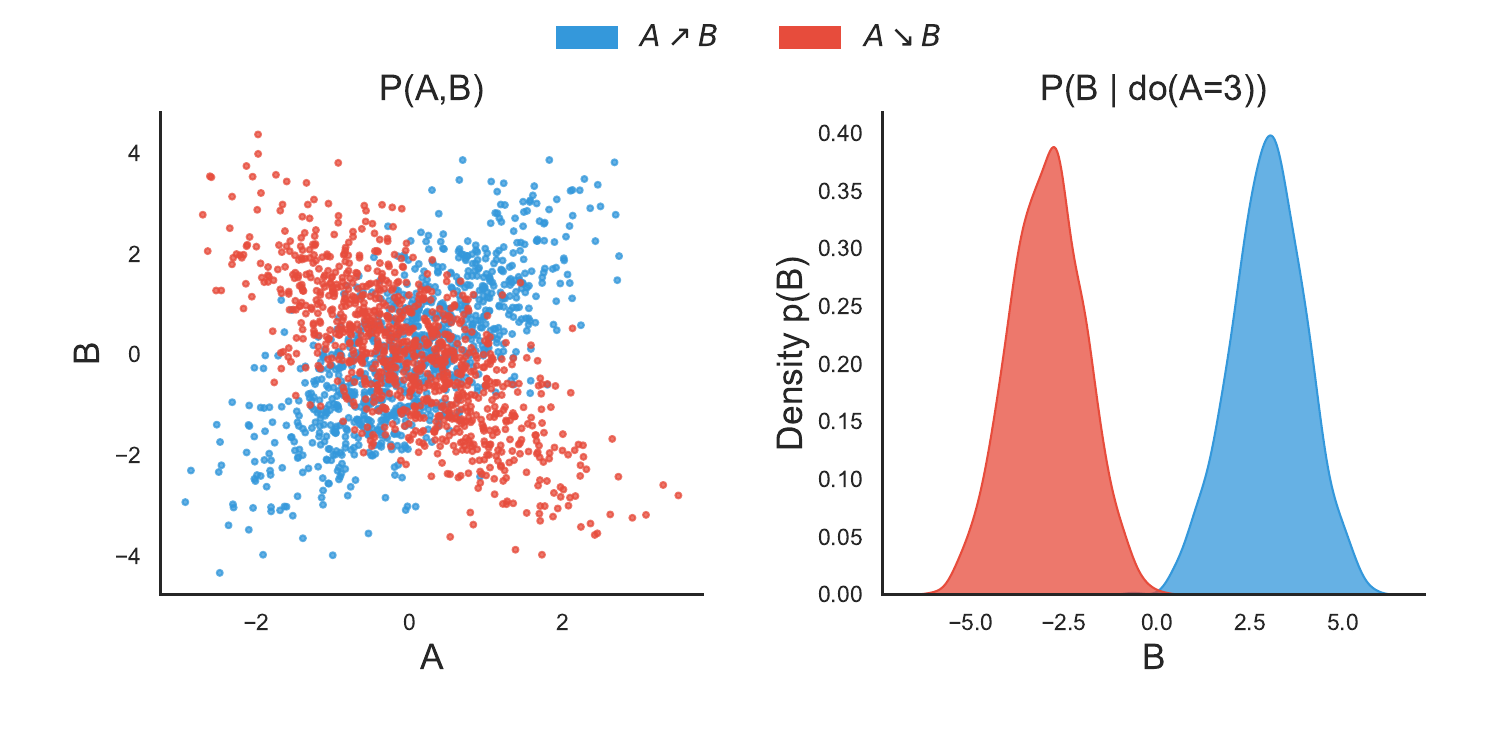}
         \caption{Observational and interventional distributions of both models with $\sigma_A = \sigma_B = 1$.}
         \label{fig:case_study_1}
     \end{subfigure}
     \hfill
     \begin{subfigure}[t]{0.48\textwidth}
         \centering
         %\captionsetup{justification=centering}
         \includegraphics[width=\textwidth]{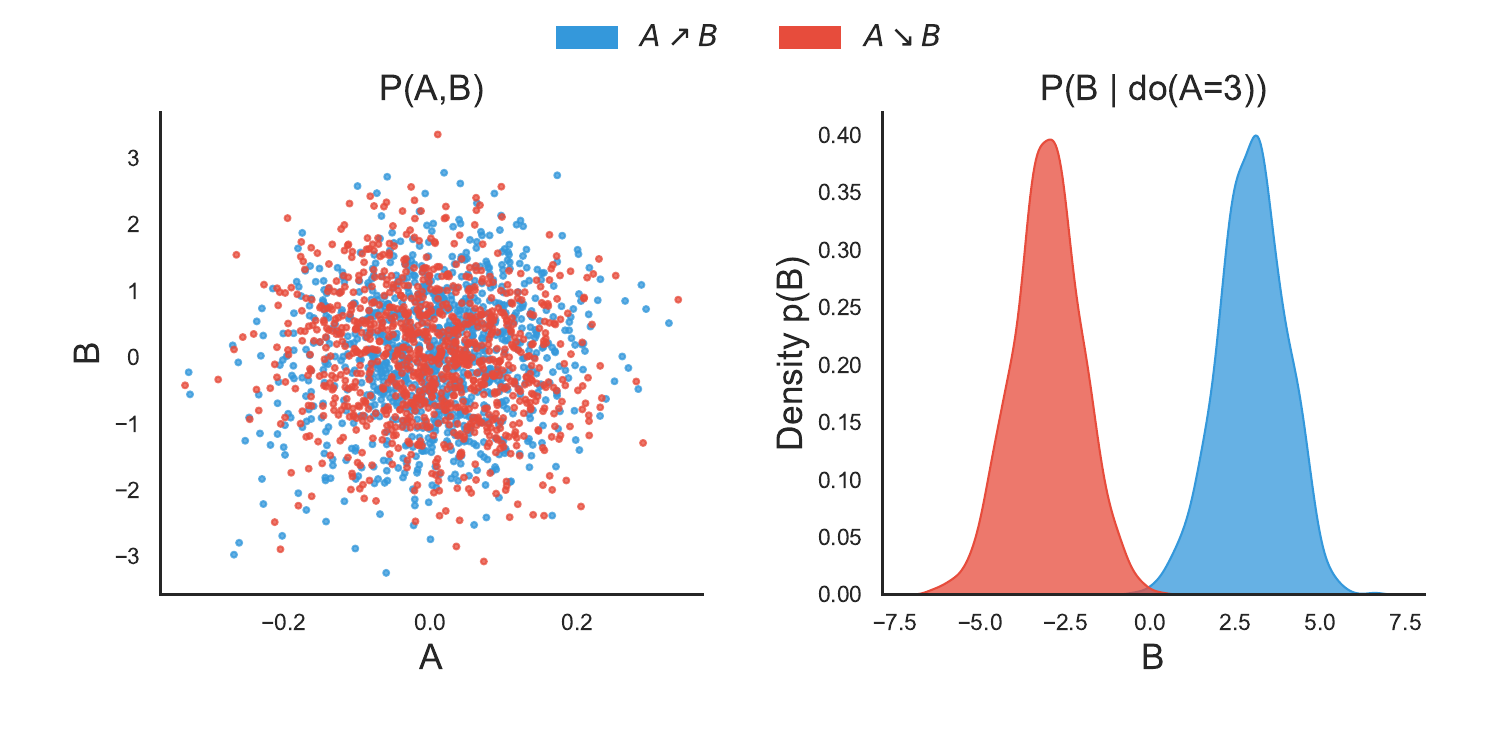}
         \caption{Observational and interventional distribution of both models with $\sigma_A = 0.1 \cdot \sigma_B$.}
         \label{fig:case_study_2}
     \end{subfigure}
     \caption{Example of two causal models $\mathfrak{C}_1$ (blue) and $\mathfrak{C}_2$ (red) with the same graph. In \Figref{fig:case_study_1}, the observational and interventional distributions are different. In \Figref{fig:case_study_2} the observational distributions are similar (due to the noise structure), but the interventional distributions remain different.}
\end{figure*}

\section{Details about the efficiency of the causal distances}
\label{app:efficiency}

In this section, we denote as $\tilde{\OD}$, $\tilde{\ID}$, and $\tilde{\CD}$ the practical estimations of the theoretical $\OD$, $\ID$, and $\CD$ and study their behavior.

\subsection{Sample efficiency}
\label{sec:sample efficiency}

\begin{figure*}[t!]
     \centering
     \begin{subfigure}[t]{0.31\textwidth}
         \centering
        %  \captionsetup{justification=centering}
         \includegraphics[width=\textwidth]{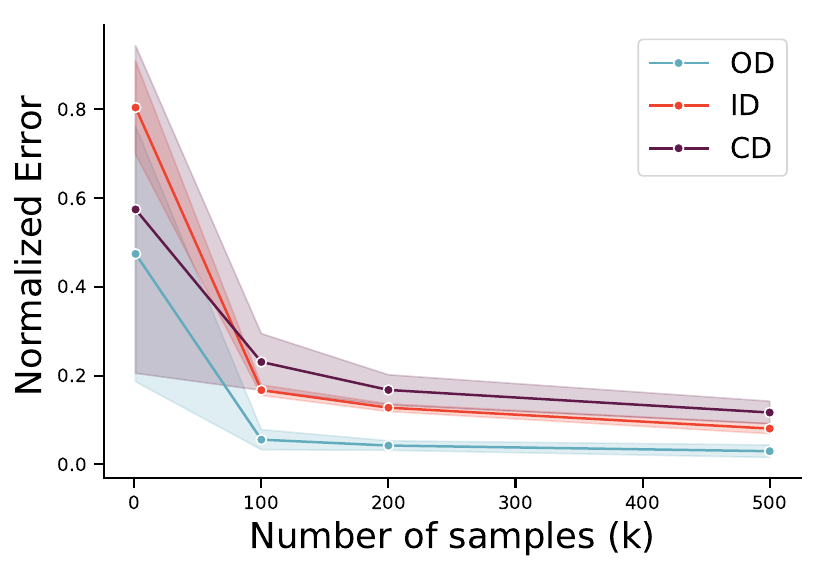}
         \caption{Sample efficiency of $\tilde{\OD}$, $\tilde{\ID}$, and $\tilde{\CD}$.}
         \label{fig:sample-efficiency}
     \end{subfigure}
     \hfill
     %\begin{subfigure}[t]{0.24\textwidth}
     %    \centering
        %  \captionsetup{justification=centering}
     %    \includegraphics[width=\textwidth]{images/Sample_efficiency-parameters.pdf}
     %    \caption{Sample efficiency of \\ $k$, $l$, and $m$ in isolation.}
     %    \label{fig:sample-efficiency-parameters}
     %\end{subfigure}
     %\hfill
     \begin{subfigure}[t]{0.31\textwidth}
         \centering
        %  \captionsetup{justification=centering}
         \includegraphics[width=\textwidth]{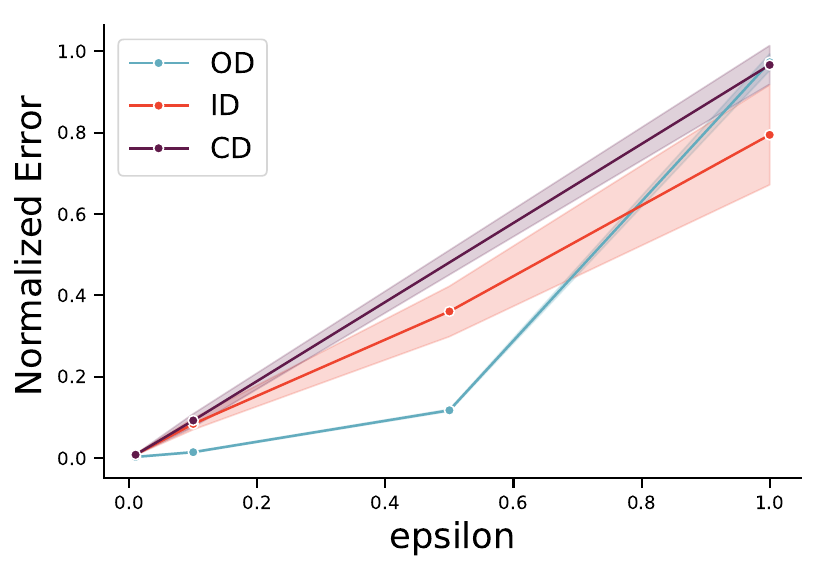}
         \caption{Sensitivity of $\tilde{\OD}$, $\tilde{\ID}$, and $\tilde{\CD}$ to perturbations.}
         \label{fig:sensitivity}
     \end{subfigure}
     \hfill
     \begin{subfigure}[t]{0.31\textwidth}
         \centering
        %  \captionsetup{justification=centering}
         \includegraphics[width=\textwidth]{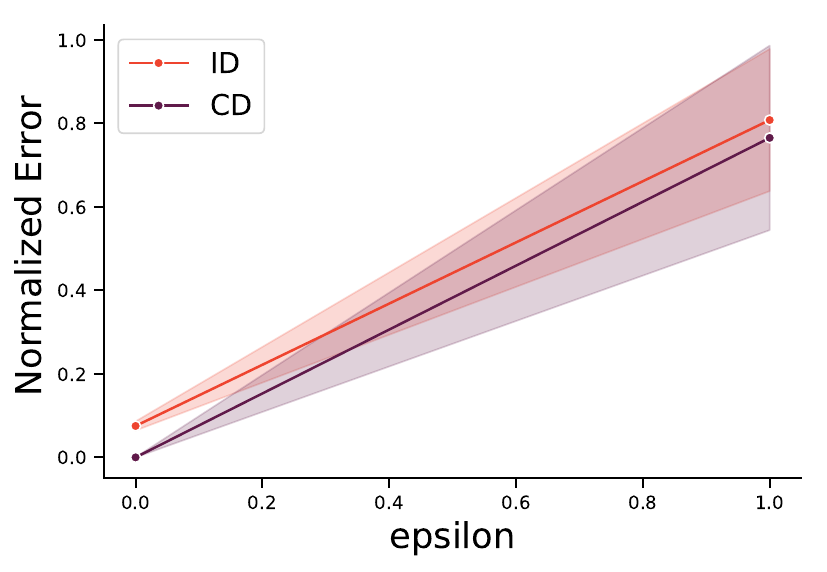}
         \caption{Sensitivity of $\tilde{\ID}$ (resp. $\tilde{\CD}$) to perturbations that leave ${\OD}$ (resp. ${\ID}$) unchanged.}
         \label{fig:sensitivity-specific}
     \end{subfigure}
     \caption{Sample efficiency and sentitivity of the proposed approximations of $\OD$, $\ID$, and $\CD$.}
\end{figure*}

We validate the sample efficiency of approximating ${\OD}$, ${\ID}$, and ${\CD}$ by observing how quickly the estimates converge to $0$ when comparing a randomly sampled causal model to itself.
% , an instance for which we know the values of $\OD$, $\ID$ and $\CD$.
%While it is possible to compute the answers analytically for some very specific causal models (linear Gaussian), comparing the causal model to itself allows to evaluate sample efficiency for arbitrary causal models.

%$\tilde{OD}$ is estimated by drawing $k$ samples from the observation distributions. $\tilde{ID}$ applies $\tilde{OD}$ for each of the $d$ interventional distributions and for $l$ values $\{do(I=i_j)\}\limits_{j=1}^l$ giving $d \cdot l \codt k$ samples. Then, $\tilde{CD}$ further requires to sample $m$ starting states $\bX = \bx$ yielding $m \cdot d \cdot l \cdot k$.

In \Figref{fig:sample-efficiency}, we fix $l= 100$ and $m = 10$, sample graphs with $d=6$ nodes, and vary the number $k$ of samples used to estimate the distributions.
The shaded area represents the standard deviation after repeating the experiment $10$ times with different random seeds.
We observe a quick decrease towards $0$ for each distance.

\subsection{Sensitivity to perturbation}
\label{sec:sensitivity}
In the next experiment, we verify that $\tilde{\OD}$, $\tilde{\ID}$, and $\tilde{\CD}$ can capture perturbations of causal models despite the imperfect approximations due to the finite sample size.

We randomly draw a causal model $\mathfrak{C}$ and perturb one of its mechanisms $f_i$ by adding another random mechanism $g_i$ according to a perturbation parameter $\epsilon \in [0,1]$. This results in a new causal model $\mathfrak{C}_{\epsilon, g_i}$ identical to $\mathfrak{C}$ except for the $i$-th mechanism, which is replaced by $(1-\epsilon) \, f_i + \epsilon \, g_i$. When $\epsilon=0$, $\mathfrak{C}_{\epsilon, g_i} = \mathfrak{C}$ and we expect the distance to be $0$. As $\epsilon$ increases, the distance should grow.

\Figref{fig:sensitivity} plots the growth of $\tilde{\OD}$, $\tilde{\ID}$, and $\tilde{\CD}$ as functions of the perturbation $\epsilon$ using $k=1000$, $l=100$ and $m=10$. Each distance increases with $\epsilon$, with the effect being more visible higher up the ladder of causation. Intuitively, slight perturbations have the potential to induce large deviations when going up the ladder. Indeed, a perturbation of the mechanism modifies the likelihood terms, which also modifies the Bayesian update of the noise variables.
%Thus, the mechanism modification also perturbs counterfactual models.
Note that graph-based metrics such as $\SHD$ and $\SID$ cannot capture these nuances because the causal graph remains unchanged.

Next, we perturb $\mathfrak{C}$ in a way that $\OD$ (resp. $\ID$) remains unchanged and observe the variations in $\tilde{\ID}$ (resp. $\tilde{\CD}$). We detail how we proceed to create such perturbations in \Appref{app:detail_sensitivity}.
We denote with $\epsilon$ the parameter that quantifies these perturbations and report the results in \Figref{fig:sensitivity-specific}, which shows that both $\tilde{\ID}$ and $\tilde{\CD}$ detect their level-specific perturbations.

\subsection{Perturbing ID while keeping OD constant}
\label{app:detail_sensitivity}
We know that if two causal graphs are within the same Markov class they can support the same observational distribution~\cite{Verma:1990}. Thus, we take a causal model $\mathfrak{C}$ with graph $\mathcal{G}$ and compute its Markov equivalence class $\mathcal{M}(\mathcal{G})$.

We then consider a causal graph $\mathcal{H} \in \mathcal{M}(\mathcal{G})$ from the Markov equivalence class and select the perturbation quantification $\epsilon$ as:
\begin{equation}
    \epsilon = \frac{\SID(\mathcal{H}, \mathcal{G})}{\max \{\SID(\mathcal{H}, \mathcal{G})| \mathcal{H} \in \mathcal{M}(\mathcal{G})\}}.
\end{equation}

Then, we train an MLE parameter estimator using $\mathcal{H}$ to find the parameters that yield (almost) the same observational distribution as $\mathfrak{C}$. Thus, $\OD$ is expected to be (almost) constant while $\ID$ is perturbated. In particular, when $\epsilon = 0$, $\ID$ is expected to be $0$.

\subsection{Perturbing CD while keeping ID constant}
%Suppose we keep the causal graph fixed.
The interventional distributions remain unchanged if all the conditional distributions $P(X|\bPA_X)$ do not change. To preserve the interventional distributions, we can perturbate the structural equations like described in ~\Secref{sec:sensitivity} to generate ~\Figref{fig:sensitivity}, but also adjust the noise to precisely cancel the perturbation and keep the conditional distribution constant.

In practice, at one node $X$, we perturbate the noise distribution by adding a random Gaussian Mixture $GMM(k, \mu, \sigma)$. Here $k$ is the number of Gaussians, $\mu$ is an $k$-dimensional vector of means and $\Sigma$ the covariance matrix.
\begin{equation}
 P^{(\epsilon)}_{N_X} := (1 - \epsilon) P_{N_X} +  \epsilon GMM(k, \mu, \Sigma)   
\end{equation}
Here, $\epsilon$ quantifies the perturbation. 
In order to preserve the conditional probability distribution $P(X|\bPA_X)$ we fit a Gaussian process $g^{\epsilon}_X$ such that:
\begin{equation}
    g^{(\epsilon)}_X(\bPA_X, P^{(\epsilon)}_{N_X}) \approx f_X(\bPA_X, P_{N_X})
\end{equation}
Thus, $\ID$ is expected to stay (almost) fixed while $\CD$ is expected to be affected because the noise is changed. In particular, when $\epsilon$ is $0$ the causal model is not modified and when $\epsilon$ is $1$ the noise is fully replaced by the random Gaussian Mixture.

\section{Finegrained analysis of causal discovery methods}
\label{app:model_perf}
We report the performance of causal discovery methods for each real-world causal model we consider in \Tabref{tab:eval-systems}.

We observe that different metrics produce different rankings of systems. This shows that the differences between metrics observed in \Secref{sec:comparison} are also visible in the causal discovery evaluation setup. 

In particular, we observe low agreement between $\SID$ and $\ID$ on the Earthquake and Insurance networks. Also, IAMB and MMPC have the same $\SID$ (16) and $\SHD$ (7) on Cancer2 but different graphs which is distinguished by $\ID$. On the contrary, on the Protein dataset, GS and IAMB have the same graph and, with fixed parameter estimation, the same $\SHD$, $\SID$ and $\ID$. 

\begin{table*}
        \small
        \centering
        \resizebox{\textwidth}{!}{
        \begin{tabular}{l|ccc|ccc|ccc|ccc|ccc|ccc|ccc}
        \toprule
        &  \multicolumn{3}{c|}{Cancer1} & \multicolumn{3}{c|}{Cancer2} & \multicolumn{3}{c|}{Child} & \multicolumn{3}{c|}{Earthquake} &  \multicolumn{3}{c|}{Insurance} & \multicolumn{3}{c|}{Protein} & \multicolumn{3}{c}{Survey} \\
         & SID & SHD & ID & SID & SHD & ID & SID & SHD & ID & SID & SHD & ID & SID & SHD & ID & SID & SHD & ID & SID & SHD & ID \\
        %                           resp.       ||      pyramid         ||      resp                            ||      pyramid
        \midrule
        LinGAM          & 38 & 14 & 5.4 & 12 & 6 & 3.44 & 282 & 45 & 4.21 & 16 & 7 & 10.56 & 528 & 91 & 5.56 & 58 & 32 & 4.44 & 26 & 8 & 1.45 \\
        %CAM          & x.xx & x.xx & x.xx & x.xx & x.xx & x.xx & x.xx & x.xx & x.xx & x.xx & x.xx & x.xx & x.xx & x.xx & x.xx & x.xx & x.xx & x.xx & x.xx & x.xx & x.xx \\
        CCDr          & 6 & 11 & 1.75 & 1 & 1 & 1.32 & 57 & 13 & 3.65 & 0 & 5 & 8. & 456 & 51 & 5.53 & 18 & 27 & 3.25 & 12 & 9 & 1.7 \\
        GS          & 18 & 6 & 1.82 & 16 & 7 & 3.7 & 273 & 44 & 3.93 & 0 & 1 & 4.82 & 542 & 64 & 5.18 & 51 & 22 & 4.31 & 27 & 11 & 1.88 \\
        GES          & 44 & 21 & 6.15 & 20 & 8 & 3.35 & 189 & 78 & 3.64 & 20 & 11 & 9.26 & 545 & 98 & 5.6 & 50 & 48 & 5.09 & 27 & 15 & 1.8 \\
        PC          & 11 & 4 & 1.54 & 12 & 6 & 2.64 & 182 & 27 & 3.84 & 0 & 1 & 4.79 & 488 & 49 & 5.17 & 40 & 20 & 4.40 & 27 & 9 & 1.81 \\
        IAMB          & 18 & 6 & 1.63 & 16 & 7 & 3.62 & 253 & 43 & 3.81 & 0 & 1 & 6.17 & 588 & 67 & 5.26 & 51 & 22 & 4.31 & 27 & 11 & 1.82 \\
        MMPC          & 51 & 16 & 5.63 & 16 & 7 & 3.19 & 367 & 61 & 4.31 & 20 & 9 & 9.13 & 682 & 97 & 5.07 & 59 & 34 & 4.35 & 27 & 11 & 1.97 \\

        \bottomrule
        \end{tabular}
        }
        \caption{Evaluation of various causal discovery techniques with SID, OD and ID.}\label{tab:eval-systems}
\end{table*}

\section{Estimating causal distances in practice}
\label{app:approximation}
We now discuss the practical computation of $\OD$, $\ID$, and $\CD$. For general causal models, they cannot be computed analytically. Instead, we must draw finitely many samples and use empirical distances $\tilde{D}$ instead of the theoretical $D$. This results in estimated distances denoted by $\tilde{\OD}$, $\tilde{\ID}$, and $\tilde{\CD}$. 
%Pseudocode and further details are available in \Appref{app:pseudo-code}.
%This gives rise to a trade-off between speed and accuracy. 

\xhdr{Observational}
In order to estimate $\OD$, we draw $k$ samples from the joint observational distribution of each model and use a sample distance $\tilde{D}$.
Consequently, the estimated $\tilde{\OD}$ directly inherits the statistical properties of the chosen estimator $\tilde{D}$ and has sampling complexity $\mathcal{O}(k)$.
    
\xhdr{Interventional}
The computation of $\ID$ involves the application of $\tilde{\OD}$ to compare $dl$ pairs of interventional distributions:
for each node $I$ from the set of all $d$ nodes, $l$ intervention values $i$ are sampled from $P_I$, and the corresponding interventional distribution $P^{\mathfrak{C}; \DO(I=i)}_{\bX}$ is estimated by drawing $k$ samples. Thus, the sampling complexity is $\mathcal{O}(dlk)$. 
% Algorithm~\ref{alg:id} is the pseudo-code for the computation of $\tilde{\ID}$.
    
\xhdr{Counterfactual}
The estimation of $\CD$ involves the computation of $\tilde{\ID}$ on several modified causal models.
For each node $E$, $m$ evidence values $e$ are sampled from $P_E$. For each evidence $E=e$, the noise distributions of both models are updated using Bayes' rule (\cf \Eqnref{eq:noise bayes}) and $\tilde{\ID}$ is computed on these modified causal models.
The sampling complexity of $\tilde{\CD}$ is therefore $\mathcal{O}(d^2 m l k)$. 
% Algorithm~\ref{alg:cd} is the pseudo-code for the computation of $\tilde{\CD}$.

The Bayesian update can be computationally demanding. To address this, we first observe that we only need to sample from $P(\bN|E=e)$ (or $P_{\bN|E=e}$ in the notation of \Eqnref{eq:noise bayes}) to estimate $\ID$ in the induced counterfactual model. Using a general Gibbs sampler, it suffices to compute the likelihood term $P(E=e|\bN=\bn)$, and thanks to the Markov factorization property, this 
reduces to
%simplifies to $P(E=e|\bN) = 
$P(E=e|\bPA_E)$. The value of $E$ is set according to the structural equation $E = f_E(\bPA_E, N_E)$. When $\bPA_E$ is given but not the noise $N_E$, we obtain a probability distribution for $E$.
% Each likelihood could then be estimated at runtime by sampling the noise $\bN$ and empirically estimating $P(E=e|\bPA_E)$ with techniques such as density estimation. To speed up the computation, we instead propose a faster alternative, as follows.

%However, in our experiments, this renders the Gibbs sampling too slow.

% If $E$ takes on values from the discrete set $\{e_1, \dots, e_n\}$, we can speed up the Gibbs sampler by simply precomputing the likelihood estimates $\{P(E =e_i |\bPA_E)\}$.
% Otherwise, if $E$ is a continuous random variable, we first discretize it and then pick $n$ evenly spaced values $\{e_1, \dots, e_n\}$ for which we precompute the likelihood estimates.
% In order to speed-up the Gibbs sampler, we precompute $n$ likelihood estimates $\{P(E =e^{(j)} |\bPA_E\}\limits_{j=1}^{n}$ for $n$ evenly spaced values $\{E = e^{(j)}\}\limits_{j=1}^{n}$.
% At runtime, when the observation $E=e$ is given, we retrieve the nearest neighbor of $e$ and use its precomputed value.
%\begin{align}
% \tilde{P}(E=e|\bPA_E) &\approx P(E=e^{*}_i|\bPA_E) \\
% \text{such that } & e^{*}_i = \argmin\limits_{e_j} |e - e_j|
%\end{align}
% In practice, the computation of $\tilde{\CD}$ is orders of magnitude faster than the na\"ive algorithm where the likelihood terms are estimated at runtime.

\xhdr{Handling of continuous input}
We require that the intervention and evidence values $i$ are drawn from a distribution with full support over $\Omega_I$~\cite{Peters2017}. In the discrete case, it is straightforward to assign a uniform distribution over the elements of $\Omega_I$. However, in the continuous case, we can use the standard Gaussian distribution. But any other user-defined perturbative distribution can be employed.

\section{Impact of more training data}
\label{app:more_data}
\begin{figure}
    \centering
    % \captionsetup{justification=centering}
    \begin{subfigure}[t]{0.49\columnwidth}
         \centering
        %  \captionsetup{justification=centering}
         \includegraphics[width=\textwidth]{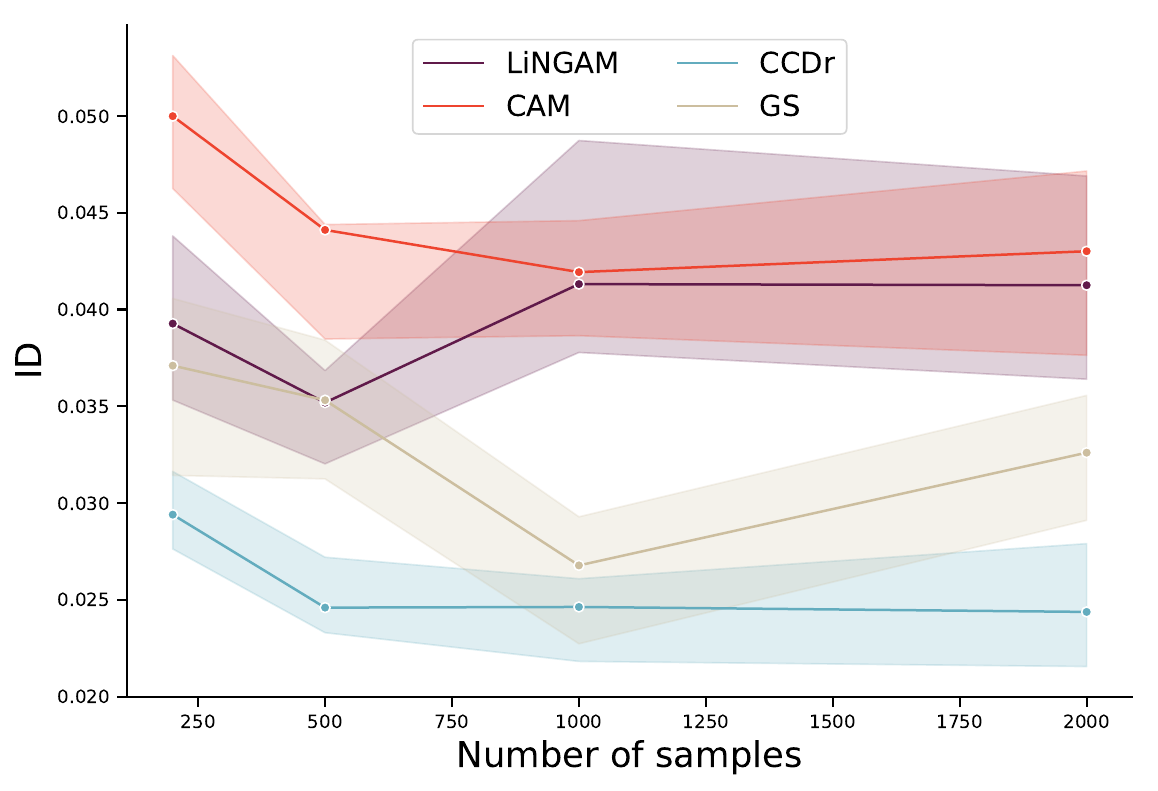}
         \caption{LiNGAM, CAM, CCDr, and GS.}
         \label{fig:samples_1}
     \end{subfigure}
     \hfill
     \begin{subfigure}[t]{0.49\columnwidth}
         \centering
        %  \captionsetup{justification=centering}
         \includegraphics[width=\textwidth]{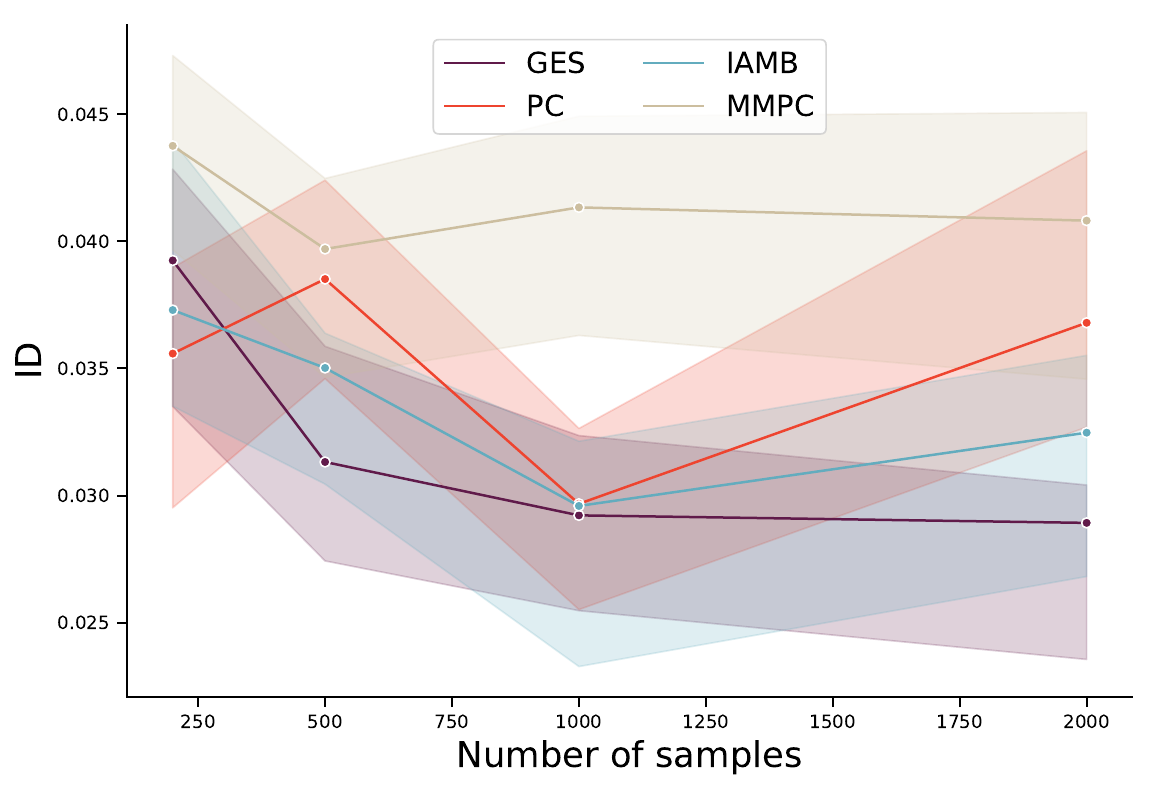}
         \caption{GES, PC, IAMB, and MMPC.}
         \label{fig:samples_2}
     \end{subfigure}
    \caption{Variation in performance of causal discovery systems measured by $\ID$ when more training data is available.}
    %\label{fig:sid-id-od}
\end{figure}

Finally, since we generated several datasets with varying number of training samples, we can measure how well systems benefit from more training data. This is reported in \Figref{fig:samples_1} and \Figref{fig:samples_2}, where the systems are arbitrarily split into two groups to avoid overcrowding one figure. Interestingly, systems seem to not clearly benefit from accessing more training data. In fact, it is a particularity of causal inference that even infinite observational data does not necessarily help to infer the causal model. 

\section{Proofs}
\label{app:proofs}

\subsection{Preliminaries}
\label{app:proof-prelim}

\paragraph{Assumptions.}
Throughout the proofs, we assume that all models satisfy the following conditions:
\begin{itemize}
\item \emph{Markov property}: 
Every conditional independence statement entailed by the causal graph is satisfied by the joint distribution: $\forall \ \bW,\bY, \bZ \in 2^{\bX}, \ \bW \independent_{\mathcal{G}} \bY|\bZ \implies \bW \independent \bY | \bZ$, where $\independent_{\mathcal{G}}$ stands for d-separated in $\mathcal{G}$~\cite{Pearl:2009}.
\item \emph{Causal minimality}:
The joint distribution satisfies the Markov property for $\mathcal{G}$ but not for any proper subgraph of $\mathcal{G}$.
\item \emph{Causal faithfulness}: 
Every conditional independence within the joint distribution is entailed by the causal graph:
$\forall \ \bW,\bY, \bZ \in 2^{\bX}, \ \bW \independent \bY|\bZ \implies \bW \independent_{\mathcal{G}} \bY | \bZ$
\item \emph{Positiveness}:
The entailed marginal and conditional distributions are strictly positive.
\end{itemize}

Then, we say that a node $X \in \bX$ \emph{has an effect on} another node $Y \in \bX$ when:
\begin{equation}
    \exists x \neq x', \ \ P^{\mathfrak{C}; \DO(X=x)}_{Y} \neq  P^{\mathfrak{C}; \DO(X=x')}_{Y}
\end{equation} 

\begin{itemize}
\item If $X$ has an effect on $Y$, then $X$ is an ancestor of $Y$. 

\item If $X$ is a parent of $Y$, then $X$ has an effect on $Y$ except if there exists a canceling path, i.e., $\exists Z_1 \dots Z_k, \ \ X \to Z_1 \to \dots \to Z_k \to Y$ which precisely cancels the effect $X \to Y$.
\end{itemize}

We order the proofs out of convenience instead of following the order in which they appear in the paper.

\subsection{Proof of Theorem~\ref{th:cd_ladder}}
\cdladder*
\begin{proof}
% Let $\mathfrak{C}_1$ and $\mathfrak{C}_2$ be two models. such that, for some $\epsilon \geq 0$:
% \begin{equation}
%     \ID(\mathfrak{C}_1, \mathfrak{C}_2) \leq \epsilon
% \end{equation}

We note $\OD_i = \OD(\mathfrak{C}_1; \DO(I=i), \mathfrak{C}_2; \DO(I=i) )$, the distance between the interventional distributions resulting from $\DO(I=i)$. 
Then, $\ID$ can be decomposed as:
\begin{equation}
    \ID(\mathfrak{C}_1, \mathfrak{C}_2) =
     \frac{1}{|\bX| + 1}\left( \OD(\mathfrak{C}_1, \mathfrak{C}_2) + \sum\limits_{I \in \bX} \mathbb{E}_{i \sim P_I} \OD_i \right) \\
\end{equation}
Since $\OD$ is a distance between distributions, $\OD_i$ is positive, and the expectations inside the sum are positive. Finally:
\begin{align}
    %  \ID(\mathfrak{C}_1, \mathfrak{C}_2) &\leq \epsilon \\
    %  (|\bX| + 1)\ID(\mathfrak{C}_1, \mathfrak{C}_2) &\leq (|\bX| + 1)\epsilon \\
    \OD(\mathfrak{C}_1, \mathfrak{C}_2) &\leq (|\bX| + 1)\ID(\mathfrak{C}_1, \mathfrak{C}_2)
\end{align}

The same reasoning gives:
\begin{equation}
    \ID(\mathfrak{C}_1, \mathfrak{C}_2) \leq (|\bX| + 1) \CD(\mathfrak{C}_1, \mathfrak{C}_2)
\end{equation}

\end{proof}

\subsection{Proof of Theorem~\ref{th:id_sid}}

\idsid*

\begin{proof}
    Suppose  $\ID(\mathfrak{C}_1, \mathfrak{C}_2) = 0$. From \Thmref{th:cd_ladder}, we know that $\OD(\mathfrak{C}_1, \mathfrak{C}_2) = 0$. 
    
    The two models belong to the same Markov equivalence class.  Thus, they have the same skeleton and v-structures~\cite{Verma:1990}. Furthermore, orienting new edges cannot create new v-structures. 
    
    Suppose some edges are still oriented differently. For example, consider the edge between $X$ and $Y$ left unoriented in the Markov equivalence class.
    Without loss of generality, suppose $X \to Y$ in $\mathcal{G}_1$.
    
    Now, $X$ has an effect on $Y$ in $\mathcal{G}_1$ because there cannot be a cancelling path. If there were a cancelling path, the orientation $X \to Y$ would create new v-structure.
    Since $X$ has an effect on $Y$ in $\mathcal{G}_1$, $X$ also has an effect on $Y$ in $\mathcal{G}_2$ (the models agree on any interventions). Finally, the edge has to go from $X$ to $Y$ in $\mathcal{G}_2$.
    
    Thus, we conclude $\mathcal{G}_1 = \mathcal{G}_2$. Finally, $\SID$ and $\SHD$ only consider the adjacency matrices and therefore they are also $0$. \newcite{SID} proved that $\SHD(\mathfrak{C}_1, \mathfrak{C}_2) = 0 \implies \SID(\mathfrak{C}_1,\mathfrak{C}_2) =0$. 
    A counterexample to the converse implication of \Eqnref{eeq:implication_1} is given by the models of the case study presented in the paper.
\end{proof}

\subsection{Proof of Theorem~\ref{th:id_sid_od}}
\sididod*

\begin{proof}
Suppose $\OD(\mathfrak{C}_1, \mathfrak{C}_2) = 0$. Then, the two models have graphs belonging to the same Markov equivalence class, i.e., same skeleton and v-structures~\cite{Verma:1990}. 

We already know from \Thmref{th:id_sid} that $\ID(\mathfrak{C}_1, \mathfrak{C}_2) = 0 \implies \SID(\mathfrak{C}_1, \mathfrak{C}_2) = 0$.

Suppose $\SID(\mathfrak{C}_1, \mathfrak{C}_2) = 0$. Then, no edge between any two nodes can be oriented differently in the two graphs. To see that, consider the edge between $X$ and $Y$ left unoriented in the Markov equivalence class. Without loss of generality, suppose $X \to Y$ in $\mathcal{G}_1$. Now, $X$ has an effect on $Y$ in $\mathcal{G}_1$ because there cannot be a cancelling path. If there were a cancelling path, the orientation $X \to Y$ would create new v-structure. Since $X$ has an effect on $Y$ in $\mathcal{G}_1$, $X$ also has an effect on $Y$ in $\mathcal{G}_2$ because the models agree on any interventions. Finally, the edge goes from $X$ to $Y$ in $\mathcal{G}_2$. Therefore, the two graphs are the same. 

As the two graphs are the same and the observational distributions are the same, we can conclude that $\ID(\mathfrak{C}_1, \mathfrak{C}_2)=0$.

A counterexample when $\OD(\mathfrak{C}_1, \mathfrak{C}_2) \neq 0$ is given by the case study presented in the paper.
\end{proof}

\end{document}